\documentclass[sn-basic,iicol]{sn-jnl}%
\usepackage{adjustbox}
\usepackage{graphicx}%
\usepackage{multirow}%
\usepackage{amsmath,amssymb,amsfonts,bm}%
\usepackage{amsthm}%
\usepackage{mathrsfs}%
\usepackage[title]{appendix}%
\usepackage{xcolor}%
\usepackage{textcomp}%
\usepackage{manyfoot}%
\usepackage{booktabs}%
\usepackage{algorithm}%
\usepackage{algorithmic}%
\usepackage{subfigure}
\usepackage{listings}%
\usepackage{soul}
\usepackage{color}
\usepackage{natbib}
\usepackage{pifont}

\definecolor{bblue}{RGB}{0,0,200}
\definecolor{blue}{RGB}{0,0,255}
\definecolor{bred}{RGB}{200,0,0}
\definecolor{bgreen}{RGB}{0,140,0}
\hypersetup{
    colorlinks=true,
    linkcolor=bred,
    citecolor=bblue,
    filecolor=mydarkblue,
    urlcolor=blue
}

\newcommand{\KL}{D_{\mathrm{KL}}}
\newcommand{\E}{\mathbb{E}}

\DeclareMathOperator*{\argmin}{arg\,min}
\newtheorem{theorem}{Theorem}
\newtheorem{lemma}{Lemma}

\begin{document}

\title[Article Title]{Article Title}

\title[Article Title]{Variational Rectification Inference for Learning with Noisy Labels}

\author[1]{\fnm{Haoliang} \sur{Sun}}\email{haolsun@sdu.edu.cn}
\equalcont{Equal contribution}
\author[2]{\fnm{Qi} \sur{Wei}}\email{1998v7@gmail.com}
\equalcont{Equal contribution}
\author[2]{\fnm{Lei} \sur{Feng}}\email{feng0093@e.ntu.edu.sg}
\author*[1]{\fnm{Yupeng} \sur{Hu}}\email{huyupeng@sdu.edu.cn}
\author[3]{\fnm{Fan} \sur{Liu}\email{liufancs@gmail.com}}
\author[4]{\fnm{Hehe} \sur{Fan}\email{hehefan@zju.edu.cn}}
\author[1]{\fnm{Yilong} \sur{Yin}}\email{ylyin@sdu.edu.cn}

\affil[1]{\orgdiv{School of Software}, \orgname{Shandong University}, \orgaddress{\city{Jinan}, \country{China}}}

\affil[2]{\orgdiv{School of Computer Science and Engineering}, \orgname{Nanyang Technological University}, \orgaddress{\country{Singapore}}}

\affil[3]{\orgdiv{School of Computing}, \orgname{National University of Singapore}, \orgaddress{\country{Singapore}}}

\affil[4]{\orgdiv{School of Computer Science and Technology}, \orgname{Zhejiang University}, \orgaddress{\country{Hangzhou}, \country{China}}}


\abstract{Label noise has been broadly observed in real-world datasets. To mitigate the negative impact of overfitting to label noise for deep models, effective strategies (\textit{e.g.}, re-weighting, or loss rectification) have been broadly applied in prevailing approaches, which have been generally learned under the meta-learning scenario. Despite the robustness of noise achieved by the probabilistic meta-learning models, they usually suffer from model collapse that degenerates generalization performance. In this paper, we propose variational rectification inference (VRI) to formulate the adaptive rectification for loss functions as an amortized variational inference problem and derive the evidence lower bound under the meta-learning framework. Specifically, VRI is constructed as a hierarchical Bayes by treating the rectifying vector as a latent variable, which can rectify the loss of the noisy sample with the extra randomness regularization and is, therefore, more robust to label noise. To achieve the inference of the rectifying vector, we approximate its conditional posterior with an amortization meta-network. By introducing the variational term in VRI, the conditional posterior is estimated accurately and avoids collapsing to a Dirac delta function, which can significantly improve the generalization performance. The elaborated meta-network and prior network adhere to the smoothness assumption, enabling the generation of reliable rectification vectors. Given a set of clean meta-data, VRI can be efficiently meta-learned within the bi-level optimization programming. Besides, theoretical analysis guarantees that the meta-network can be efficiently learned with our algorithm. Comprehensive comparison experiments and analyses validate its effectiveness for robust learning with noisy labels, particularly in the presence of open-set noise. }


\maketitle

%\clearpage
\section{Introduction}\label{sec:introduction}
Learning from noisy labels (LNL) \citep{fu2024noise, xia2023combating,yuan2023late,huang2023paddles,wei2023fine,xu2021faster,ortego2021multi,gudovskiy2021autodo} poses great challenges for training deep models, whose performance heavily relies on large-scaled labeled datasets. Annotating training data with high confidence would be resource-intensive, especially for some domains with ambiguous labels, such as medical image segmentation and re-identification tasks \citep{pu2023memorizing, liu2023mitigating}. In this case, label noise would inevitably arise since there is usually a lack of experts for accurate annotation. 

Re-weighting~\citep{kumar2010self,zadrozny2004learning,jiang2018Mentornet,shu2023cmw} and loss rectification~\citep{zhangyiwan2021learning,vahdat2017toward,Yu2020_dual} are two effective strategies to reduce the bias of learning caused by noisy labels. The basic idea is to construct a weight function or transition matrix to mitigate the effect of noisy samples. Although those strategies have been broadly applied, there are two limitations. 1) The form of the weighting functions needs to be manually specified under certain assumptions on the data distribution, restricting its expandability in the real world~\citep{shu2019meta}. 2) Hyper-parameters in these functions are usually tuned by cross-validation, which suffers from the issue of scalability~\citep{franceschi18a}.

A family of approaches based on meta-learning has been recently proposed for noisy labels~\citep{shu2023cmw,xu2021faster,zheng2021meta,zhang2019metacleaner,shu2019meta,zhao2021probabilistic,sun2021learning,yichen2020softlabel}. By introducing a small meta-data set with completely clean labels, an effective weighting (\textit{e.g.}, meta-weight-net~\citep{shu2019meta}) or correction (\textit{e.g.}, meta label correcter~\citep{zheng2021meta}) function can be meta-learned under the meta-learning scenario, omitting the prior assumption for these functions and avoiding manually tuning of hyper-parameters~\citep{ren2018learning}. %Therefore, learning an effective meta function can significantly improve robustness of the learned prediction model
To enhance the interpretability and generalization ability, Bayesian meta-learning~\citep{zhao2021probabilistic,sun2021learning} has been applied to model the uncertainty of parameters and achieved a favorable performance for learning with noisy labels. The probabilistic meta-weight-net~\citep{zhao2021probabilistic} applies a Bayesian weight network to estimate the distribution of the sample weight. The probabilistic formulation is elegant. However, the weighting network merely takes the loss as the input to compute the sample weight, it would be deficient in controlling the learning process and result in low expression capability~\citep{sun2021learning}. To strengthen the capability of the meta-network, a rectification network has been proposed in~\citep{sun2021learning} to achieve rectifying the training process with an estimated vector. By treating the rectifying vector as a latent variable, the predictive posterior can be estimated by Monte-Carlo (MC) approximation. %Although the MC approximation has achieved desirable effectiveness for rectifying the bias of the learning process, we have observed that there would exist model collapse~\citep{iakovleva2020meta} where the conditional prior collapses to a Dirac delta function and the model degenerates to a deterministic parameter generating network, especially for a small sampling number in MC. This collapse may degrade the generalization performance of the model. %, which is illustrated in Figure \ref{fig:intro}.
Although the probabilistically formulated LNL rectification method has demonstrated effectiveness, there are two issues in existing methods: 1) Model collapse, where it has been observed that the conditional prior may collapse to a Dirac delta function, and the model degenerates to a deterministic parameter-generating network, especially for a small sampling number in MC~\citep{iakovleva2020meta}. This collapse can degrade the model's generalization performance. 2) Overlooking the intrinsic smoothness assumption in data, where the meta-network should be primarily aware of discriminative information from the feature rather than rely on the potentially noisy label, especially when the discriminative information in the feature is inconsistent with the label.

%%%%%%%%%%%%%%%%%%%%%%%%%%%%%%%%%%%%%%%%%%%%%%%%%%%%%%%%%
\iffalse

\begin{figure}[t]
\centering
\includegraphics[width=1.0\linewidth]{ intro.pdf} 
\caption{Illustration of model collapse with 70\% uniform noise. There exists a gap between MC and VRI in the meta-loss curve. The norm of variance for the rectification vector of MC degenerates into zero in some cases. The generalization performance is also degraded.}
\label{fig:intro}
\end{figure}

\fi
%%%%%%%%%%%%%%%%%%%%%%%%%%%%%%%%%%%%%%%%%%%%%%%%%%%%%%%%%%

In this work, to tackle the two issues in existing works, we propose to formulate learning rectification process as an amortized variational inference problem and derive the evidence lower bound (ELBO) under the meta-learning framework. We construct variational rectification inference (VRI) to achieve an adaptively rectifying learning process for noisy labels as shown in Figure \ref{fig:demo}. We treat the rectifying vector as a latent variable and build a hierarchical Bayes under the setting of the meta-learning scenario. We introduce an amortization meta-network to estimate the posterior distribution of the rectifying vector and achieve a rectified prediction via Monte Carlo sampling. 
The proposed meta-network is built to leverage the feature embedding and corresponding label as inputs, which can faithfully exploit sufficient information lying in the feature space and significantly improve the generalization performance of the classification network. 

By building a variational term with a prior network to constraint the posterior, VRI can avoid the model collapse in MC approximation with limited samples and further enhance the capability of inference for the unbiased estimation of the predictive posterior. By incorporating a prior network with the input of feature embedding and minimizing the variational term, the meta-network can effectively acquire discriminative knowledge from the feature and generate reliable rectification vectors that adhere to the smoothness assumption. VRI can be integrated into the meta-learning framework to achieve adaptive rectification for noisy samples. By introducing the meta-data, we conduct the meta-learning process with a bi-level programming schema and achieve robust learning with label noise. %Unlike those label correction methods, VRI employs a vector to rectify the learning process of the classification network, enabling it to handle open-set label noise effectively.

\begin{figure}[t]
\centering
\includegraphics[width=0.48\textwidth]{ 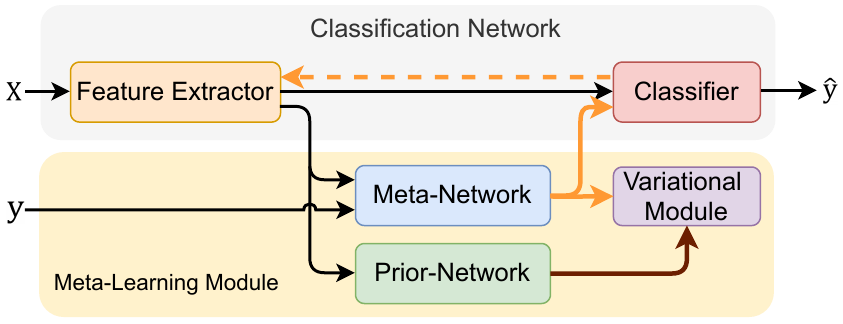}
\vspace{-3mm}
\caption{The meta-network can generate the rectifying vector to integrate into the inference of the classification network. The variational module can avoid model collapse via a prior network.}
\label{fig:demo}
\end{figure}

Our contributions can be summarized in four aspects. 
\begin{itemize}
\setlength{\itemsep}{0pt}
\setlength{\parsep}{0pt}
\setlength{\parskip}{0pt}
  \item  We formulate the learning rectification process as an amortized variational inference problem and derive a new versatile ELBO objective for LNL in the context of meta-learning. 
   \item We introduce a regularization term based on KL divergence, which can facilitate the development of a non-degenerate prior and prevent model collapse in MC approximation.
   \item We elaborate a meta-network and prior network that adhere to the smoothness assumption, enabling the generation of reliable rectification vectors.  
   \item We propose a new bi-level optimization algorithm for the objective and provide a theoretical guarantee for its convergence rate.

\end{itemize}

We carry out comprehensive experiments on five challenging benchmark datasets with various types of noise. Our VRI method surpasses state-of-the-art approaches in most scenarios, especially when dealing with open-set noise. Further promising results and additional complementary analysis also underscore the effectiveness of VRI.

The rest of this paper is organized as follows. Section \ref{sec:related} introduces related works and discusses the relations to our work. Section \ref{sec:method} includes the problem setting, preliminaries, and our noise-robust method Variational Rectification Inference (VRI). We also provide the theoretical provide for VRI. Section \ref{sec:experiments} reports experimental results on three noise types and various datasets. We test the performance of VRI on the restricted scenario (\textit{i.e.} training without the meta-set).  Finally, Section \ref{sec:conclusions} gives a conclusion.

\section{Related Work}\label{sec:related}
\textbf{Re-weighting}. The main idea of the sample re-weighting strategy is to assign a small weight to samples with corrupted labels \citep{shu2019meta}. \citep{liu2015classification} provide a theoretical guarantee that any surrogate loss function can benefit from the importance reweighting of samples and propose a new strategy to estimate the noisy ratio. Since the clean example usually has a small loss and deep models can memorize them at the beginning of the training steps~\citep{arpit2017closer}, samples with the lower loss are selected for learning at each epoch in~\citep{shen2019learning,cui2019class}. 
Based on this assumption, MentorNet~\citep{jiang2018Mentornet} adopts the idea of curriculum learning to train a mentor network to guide the learning of the student classification network.
A Bayesian model~\citep{wang2017robust} has also been extended to infer the latent variables of sample weights for handling label noise. 
To avoid manually designing or tuning weighting functions, meta-learning has been introduced to learn to generate weights from a meta-data set with clean labels. The pioneering work, inspired by the two nested loops of optimization~\citep{finn2017model}, sets the weight value as trainable parameters~\citep{ren2018learning} and achieves a dynamic weighting strategy. Meta-Weight-Net~\citep{shu2019meta} further improves the scalability of the weighting space by directly generating weights via an MLP and being learned under the meta-learning scenario.

\textbf{Correcting}. There are plenty of methods working on loss or label correction of the objective function, which can be essentially categorized into three aspects. 1) A confusion matrix~\citep{sukhbaatar2015training, bo2018masking, tanno2019learning, Yu2020_dual}, restoring the transition probability between the true label and the noisy one, is estimated and multiplied to the prediction vector. This can be considered as a smooth regularization for the prediction to mitigate the impact of corrupted labels. The following works~\citep{hendrycks2018using,pereyra2017regularizing} introduce a set of clean anchor-data to improve the estimation accuracy of the confusion matrix. Recently, an MC approximation framework~\citep{sun2021learning} is proposed to learn to generate the rectification vector for loss functions, demonstrating the superiority of handling the sample ambiguity in noisy data.
2) Another family of methods, such as Reed~\citep{reed2014training}, D2L~\citep{ma2018dimensionality}, S-Model~\citep{goldberger2016training}, includes extra inference steps to correct corrupted labels for the following optimization. By leveraging clean meta-data, MSLC~\citep{yichen2020softlabel,zheng2021meta} learns an efficient label corrector to reduce label noise. 
3) Designing appropriate loss functions also provides an effective solution to significantly enhance the robustness of deep models. Noise-tolerant losses, such as mean absolute error (MAE), have been theoretically analyzed for noisy labels in~\citep{ghosh2017robust}.
The following works~\citep{zhang2018generalized,wang2019improving} further improve the performance of MAE on challenging datasets with generalized MAE and cross-entropy losses. 
Recently, a dynamically weighted bootstrapping loss~\citep{arazo2019unsupervised} has been designed for noisy samples based on an unsupervised beta mixture model.

\textbf{Noise confusion matrix}. The confusion matrix is commonly used for label correction \citep{cheng2022class,li2022estimating,yao2021instance}. The key aspect of confusion matrix methods involves estimating the confusion matrix $T$, which largely depends on the estimated noisy class posterior. To mitigate the negative impact of inaccurate noisy class posterior estimates on $T$, \citep{cheng2022class} introduced a forward-backward cycle-consistency regularization for improved estimation of the confusion matrix.  \citep{yao2021instance} employed a causal graph to enhance the identifiability of matrix $T$, aiding in the inference of clean labels. Additionally, the study proposed by \citep{li2022estimating} explores noisy multi-label learning and introduces a novel estimator that utilizes label correlations effectively, performing well without the need for anchor points or precise fitting of the noisy class posterior.

\textbf{Meta-learning}, leverages shared knowledge among a series of tasks to improve the performance of the current task, which has made great breakthroughs recently~\citep{hospedales2021meta}. The typical idea is to parameterize a trainable function as the meta-learner to generate the parameters or statistics for base learners, which can be regarded as the "black-box" adaptation. By introducing the clean meta-data set, the aforementioned strategies (\textit{e.g.}, re-weighing~\citep{ren2018learning,zhao2021probabilistic,shu2019meta,shu2023cmw} or loss correction~\citep{zhang2019metacleaner,yichen2020softlabel,zheng2021meta,sun2021learning}) can be meta-learn in a data-driven way, avoiding manually tuning hyper-parameters with the validation set in conventional methods~\citep{ren2018learning}. \citep{taraday2023enhanced} developed a teacher-student model that adheres to an advanced bi-level optimization process. Specifically, they formulated a more precise meta-gradient for teacher learning, while the teacher network produces refined soft labels for the student. \citep{zhang2021learning} crafted a meta-based re-weighting framework, introducing historical proxy reward data to lessen dependency on clean meta-data and employing feature sharing to decrease optimization costs. \citep{kye2022learning} combined the estimation of the transition matrix with a meta-optimization framework, facilitating label correction and enabling single back-propagation through a dual-head architecture.

\textbf{Variational inference}. In practical implementations, stochastic Monte Carlo or analytic approximations are commonly used methods in Bayesian models \citep{murphy2023probabilistic}. The Variational Auto-Encoder (VAE) \citep{kingma2013auto}, a pioneering generative model, employs variational inference (VI) \citep{shen2019unsupervised} for learning directed graphical models, achieving notable advancements in image generation \citep{zhu2022disentangled} and disentangled representation \citep{higgins2016beta}. Its application extends to weakly-supervised learning tasks. \citep{wang2017robust} treat example weights as latent variables and utilize automatic differentiation variational inference for weight inference, facilitating noise-tolerant learning through a reweighting strategy. The Probabilistic Meta-Weight-Net \citep{zhao2021probabilistic} employs a Bayesian weight network to estimate the distribution of sample weight and formulates the objective as a VI problem. Another advantage of the VI formulation is that it avoids the model collapse observed in the MC method, where the conditional prior of the parameter tends to collapse to a Dirac delta function when using a small number of samples for stochastic back-propagation \citep{iakovleva2020meta}.

\textbf{Semi-supervised learning} (SSL), builds a labeled set that contains confident examples by sample selection strategies and employs modern SSL techniques (e.g., FixMatch \citep{sohn2020fixmatch}, MixMatch \citep{berthelot2019mixmatch}, and other methods~\citep{Yang10536663}) to effectively leverage the labeled set and the remaining unlabeled set \citep{li2020dividemix,liu2020early,wei2020combating}. Compared with other branches, SSL-based methods have achieved state-of-the-art performance on image benchmarks since they can incorporate prior knowledge to exploit discriminative information from finite training samples. \textit{However, the data generative process has an impact on the performance of SSL methods} \citep{pmlr-v202-yao23a}. \textit{When the image feature is the cause of the label, the performance of SSL methods is worse than model-based methods, e.g., the method based on the confusion matrix} \citep{Yu2020_dual}.

%To be specific, given a selection criterion (e.g., small-loss), the label-corrupted training set would be split into a clean set and an unreliable (noisy) set. Then, unreliable samples' labels are discarded and subsequently input into a semi-supervised training framework (e.g., FixMatch \citep{sohn2020fixmatch} and MixMatch \citep{berthelot2019mixmatch}).

\vspace{1mm}
\textbf{Other methods}. 
Additional lines of methods for handling label noise include 1) data augmentation~\citep{zhang2017mixup, Nishi_2021_CVPR}, exploring different augmentation policies to mitigate the side-effect of noisy labels, 2)
sample selection~\citep{yu2019does,wei2022self,xia2023combating}, designing an effective selection strategy to select clean data from the noisy training set, 3) model regularization, ~\citep{liu2020early,kang2023unleashing}, combating noisy signal by regularizing model in the learning stage. For example, \citep{kang2023unleashing} reveal that integrating widely adopted regularization strategies, such as learning rate decay, model weight averaging, and data augmentation, can surpass the performance of state-of-the-art methods. 4)
Contrastive learning~\citep{wei2023fine,li2022selective}, combating noisy signal via enhancing representation ability of deep models. 

Recently, studies \citep{yao2021jo, sun2022pnp, xu2023usdnl, kang2023unleashing} have concentrated on open-set label noise, where the training set includes out-of-distribution samples. For instance, Jo-SRC \citep{yao2021jo} and PNP \citep{sun2022pnp} propose a method based on consistency to identify open-set examples and then mitigate their impact by removing them. USDNL \citep{xu2023usdnl} estimates the uncertainty of network predictions after early training using single dropout, then incorporates this into the cross-entropy loss and selects samples using the small-loss criterion.

\vspace{1mm}
\textbf{Relations to us}. In contrast to prevailing works, we formulate the rectification process as an amortized variational inference problem. By building a hierarchical Bayes model, VRI exhibits the favorable property of handling the sample ambiguity. The variational term in VRI can avoid the model collapse existing in those MC approximation methods. Unlike those label correction methods, our method, VRI, employs a vector to rectify the learning process of the classification network, enabling it to handle open-set label noise effectively.

\section{Method}\label{sec:method}

We propose variational rectification inference (VRI) for adaptively rectifying the learning processing under the setting of meta-learning, which effectively mitigates the side-effect of noisy labels. VRI includes a meta-network that generates a rectifying vector to support the learning of the classification network. The whole learning procedure is formulated as an amortized variational inference problem. We integrate VRI into the bi-level optimization steps and achieve meta-learning in the rectifying process. 

\subsection{Preliminaries}
\noindent\textbf{Robust Learning with Meta-Data}. 
Given the training set $\mathcal{D}_N = \{\mathbf{x}^{(i)},\mathbf{y}^{(i)}\} _{i = 1}^N$ with noisy labels, the aim for robust learning is to achieve good generalization performance on the clean testing set.  
Under the setting of meta-learning, we construct a set of clean examples $\mathcal{D}_M = \{ \tilde{\mathbf{x}}^{(i)},\tilde{\mathbf{y}}^{(i)}\} _{i = 1}^M$, regarded as the meta-data set, which is smaller than the training set $\mathcal{D}_N$ of $N \gg M$. We usually choose the validation set as the meta-data set in practice. Therefore, the meta-learning process can be considered as learning to tune the hyper-parameters in a data-driven way.

\vspace{1mm}
\noindent\textbf{Rectification for the loss function}. Loss rectification~\citep{hendrycks2018using, sun2021learning} is an effective tool for mitigating the effect of the label noise with meta-data. 
There are essentially two networks in the learning framework. The meta-network $V(\mathbf{y}^{(i)}, \mathbf{z}^{(i)}; \phi )$ with the parameter of $\phi$ is trained with the meta-data set to take the feature embedding $\mathbf{z}^{(i)}$ and label $\mathbf{y}^{(i)}$ of the example $(\mathbf{x}^{(i)},\mathbf{y}^{(i)})$ as input and generates a vector $\mathbf{v}^{(i)}$ to rectify the learning process of the classification network. Let $\odot$ denote the element-wise product. By multiplying $\mathbf{v}^{(i)}$ on the logits calculated from the classification network $\mathbf{v}^{(i)}  \odot F(\mathbf{x}^{(i)};\theta)$, the rectified loss with noisy labels can still produce effective update direction. Therefore, the negative impact from corrupted labels in the noisy training set can be mitigated.

\subsection{Variational Rectification Inference}
The inference process in our framework is built as a hierarchical Bayes model. From the probabilistic perspective, we treat the rectifying vector as the latent variable and compute the posterior distribution $p(\mathbf{v} | \mathbf{x},\mathbf{y})$ given the observation of the sample. Our goal of this task is to accurately approximate the conditional predictive distribution with parameters $\theta$ by maximizing its log-likelihood
\begin{equation}\label{eq:ppd}
\max \, \log \,  p_{\theta}(\mathbf{y}|\mathbf{x})  = \log \int p_{\theta}(\mathbf{y}| \mathbf{x},\mathbf{v})p(\mathbf{v}| \mathbf{x})\text{d}\mathbf{v}.
\end{equation}

The rectified learning process in this work consists of two steps. First, form the posterior distribution $p(\mathbf{v}| \mathbf{x})$ over $\mathbf{v}$ for each sample $(\mathbf{x},\mathbf{y})$. Then, calculate the posterior predictive $p_{\theta}(\mathbf{y} | \mathbf{x}, \mathbf{v})$. Since inferring the posterior $p(\mathbf{v} | \mathbf{x})$ is generally intractable, we resort to approximating it by leveraging a variational distribution $q_\phi(\mathbf{v}| \mathbf{x},\mathbf{y})$. We minimize the Kullback–Leibler (KL) divergence $\KL$ between $q_\phi(\mathbf{v}| \mathbf{x},\mathbf{y})$ and $p(\mathbf{v} | \mathbf{x},\mathbf{y})$ to obtain the variational distribution
\begin{equation}\label{eq:kl}
\min  \KL[q_\phi(\mathbf{v}| \mathbf{x},\mathbf{y}) || p_{\theta}(\mathbf{v} | \mathbf{x},\mathbf{y})].
\end{equation}

We can then derive the tractable evidence lower bound (ELBO) of the conditional predictive distribution to approximate the posterior $p(\mathbf{v} | \mathbf{x},\mathbf{y})$ by applying the Bayes' rule
\begin{equation}
\label{eq:elbo}
\begin{aligned}
  &\max \log \, p_{\theta}(\mathbf{y}|\mathbf{x}) \geq \mathcal{L}_{\text{ELBO}} \\ & = \mathop{\E}_{q_\phi(\mathbf{v}| \mathbf{x},\mathbf{y})} p_{\theta}(\mathbf{y}| \mathbf{x},\mathbf{v}) - \KL[q_\phi(\mathbf{v}| \mathbf{x},\mathbf{y}) || p_\omega(\mathbf{v}|\mathbf{x})].
\end{aligned}
%\vspace{-5mm}
\end{equation}

The first term of the ELBO is the predictive log-likelihood conditioned on the input $\mathbf{x}$ and the inferred rectifying vector $\mathbf{v}$. Maximizing it can achieve accurate rectified prediction for each sample. The second term is to minimize the discrepancy between the variational distribution $q_\phi(\mathbf{v}| \mathbf{x},\mathbf{y})$ and the prior $p_\omega(\mathbf{v}|\mathbf{x})$ assigned to a certain distribution form. The detailed derivation of the ELBO is provided in Appendix \ref{de_ELBO}. Once we obtain $q_\phi(\mathbf{v}| \mathbf{x},\mathbf{y})$, the inference procedure can be summarized as 1) forming the variational distribution $q_\phi(\cdot)$ on the fly with amortized variational inference (AVI); 2) calculating the posterior predictive distribution $p(\mathbf{y}|\mathbf{x}, \mathbf{v})$ via Monte Carlo estimation. 

From a learning perspective, the regularization term facilitates the development of a non-degenerate prior and prevents model collapse. Specifically, the KL term $\KL[q_\phi(\mathbf{v}| \mathbf{x},\mathbf{y}) || p_\omega(\mathbf{v}|\mathbf{x})]$ measures the divergence between the posterior $q_\phi$ and the prior $p_\omega$. Consider the scenario in MC approximation where the posterior $q_\phi$ converges to a Dirac delta, with the Gaussian variance approaching zero, while the prior remains unchanged. In this case, the KL divergence becomes significantly large and penalizes this situation, thereby preventing the model from becoming deterministic.

From a practical perspective, generating the reliable rectification vector can be aware of the \textit{smoothness assumption}. The purpose of the regularization term is to minimize the distance between generated distributions of the rectification vector given different inputs of $(\mathbf{x}, \mathbf{y})$ and $\mathbf{x}$. In other words, the output of the posterior given the feature and label $(\mathbf{x}, \mathbf{y})$ should be similar to the output of the prior given the feature $\mathbf{x}$. This implies that the meta-network should be primarily aware of discriminative information from the feature $\mathbf{x}$ rather than the potentially noisy label $\mathbf{y}$, when the discriminative information in the feature is inconsistent with the label. We can conclude that the smoothness assumption can guide reliable rectification.

\vspace{1mm}
\noindent\textbf{Application Details}.
In practice, we assume that the latent variable $\mathbf{v}$ obeys the factorized Gaussian distribution $\mathbf{v} \! \sim \! \mathcal{N}(\bm{\mu}, \bm{\sigma}^2)$. There are three networks in our framework. The classification network $F$ with parameters $\theta$ works on the basic categorizing task. We implement the variational distribution with an amortization meta-network $V$ with parameters $\phi$ that takes a pair of the feature embedding and label of the sample as input and outputs the parameters $(\bm{\mu}, \bm{\sigma}^2)$ of a factorized Gaussian distribution $q$. By sampling a vector $\mathbf{v}^{(i)}$ from $q$, $F$ can compute a rectified prediction $\hat{\mathbf{y}}^{(i)}$. The prior is also implemented as a network $H$ with parameters $\omega$ that takes the feature as inputs and outputs another factorized Gaussian distribution $p$. To enable an unbiased estimate of the objective in Eq. (\ref{eq:ppd}), we adopt the Monte Carlo Sampling strategy that repeats the above process multiple times and averages all predictions. Note that it is commonly intractable to back-propagate through sampling operations, we solve it by applying the reparameterization trick proposed in~\citep{kingma2013auto} as
\begin{equation}\label{eq:reparam}
\mathbf{v}=\bm{\mu} + \bm{\sigma} \cdot \bm{\epsilon} \quad \text{with} \;\; \bm{\epsilon} \sim \mathcal{N}(0 ,\text{I}).
\end{equation}
We denote $\text{RP}(\cdot)$ as the sampling operation with the reparameterization trick for simplicity in the following section.

\begin{figure}[t]
\centering 
\includegraphics[width=0.48\textwidth]{ 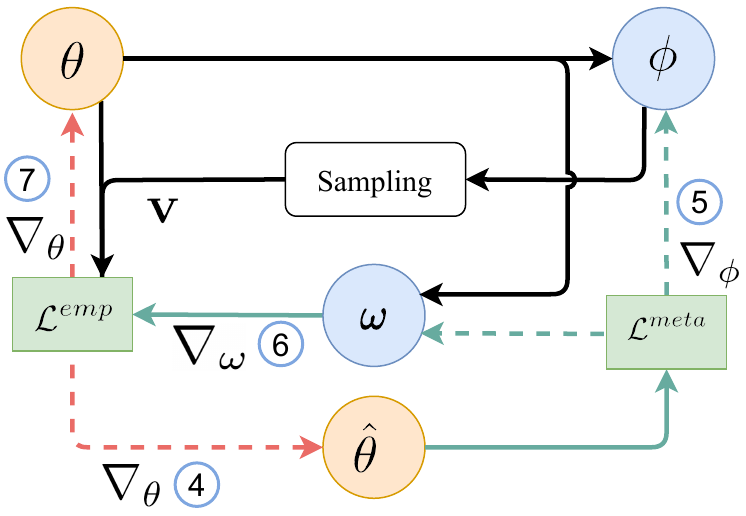} \vspace{-3mm}
\caption{Flowchart of the learning algorithm. The solid and dashed lines denote forward and backward propagation, respectively. For each iteration, the meta-network $\phi$ generates the distribution of $\mathbf{v}$ and then produces multiple examples via the sampling module to estimate the predictive distribution. By computing the gradient through the update step 4, the meta-network can be trained in step 5. The prior network is also jointly optimized in step 6. The classification network $\theta$ will be updated with support of the learned meta-network in step 7.}  
\label{fig:fig3}
\end{figure}

\subsection{Meta-Learning Process}
We present the practical objective function to achieve jointly learning the three networks of $F_\theta(\cdot)$, $V_\phi(\cdot)$, and $H_\omega(\cdot)$. By formulating the problem as a meta-learning task, we conduct bi-level optimization programming to solve it. The exhaustive derivation for each updating step is also provided in the following. 

\subsubsection{The practical objectives}
We derive the practical objective from the ELBO in Eq. (\ref{eq:elbo}). To improve the generalization performance on noisy labels, the empirical loss for our prediction model $F(\cdot)$ of $N$ samples is rectified with the support of the meta-network 
\begin{equation}\label{eq:obj0}
    \mathcal{L}^{emp}(\theta) = \frac{1}{N}\sum\limits_{i = 1}^N  L(\mathbf{y}^{(i)},\pi [ \mathbf{v}^{(i)} ] \odot F_\theta(\mathbf{x}^{(i)})), 
\end{equation}
where $\mathbf{v}^{(i)}$ is a rectifying vector sampled from the variational posterior $q^{(i)}(\mathbf{v})$ with the form of the factorized Gaussian $\mathcal{N}(\bm{\mu}^{(i)}, {\bm{\sigma}^{(i)}}^2)$, whose parameters are generated by the amortization meta-network $(\bm{\mu}^{(i)}, {\bm{\sigma}^{(i)}}) \leftarrow V_\phi(F'_{\theta'}(\mathbf{x}^{(i)}),\mathbf{y}^{(i)})$. $F'_{\theta'}$ is the feature extractor in $F_\theta$, where $\theta' \subset \theta$. 
Since elements in the rectification vector, when sampled from a Gaussian distribution, could be negative and thus disrupt the learning process, potentially leading to degraded performance, we adopt an alternative solution by constraining it to the [0,1] interval with a sigmoid function $\pi$. The choice of constraining function is further analyzed in the Experiments section.
The form of the loss function $L(.)$ is flexible, we adopt the basic cross-entropy loss with the softmax function. 

% Recall the reparameterization trick in Eq. (\ref{eq:reparam}), since we assume that the variable $\mathbf{v}$ obeys a factorized Gaussian distribution $\mathbf{v}\sim N(\mu ,\sigma^2)$, we adopt the reparameterization trick proposed in~\citep{kingma2013auto} to perform back-propagation of the sampling operation as

For the objective \textit{w.r.t.} $F_\theta$, the aim is to achieve the unbiased estimation of the conditional predictive distribution, which can be attained with Monte Carlo sampling. Recall reparameterization (RP) in Eq. (\ref{eq:reparam}), supposing the sampling number for $\mathbf{v}$ is $k$, the ultimate objective for the ELBO in Eq. (\ref{eq:elbo}) can be written as 
\begin{equation}\label{eq:obj1}
\small
 \begin{aligned}
    & \mathop{\argmin}_{\theta} \, \mathcal{L}^{emp}(\theta) = \\ &  \;  \frac{1}{k N}\sum\limits_{i = 1}^N \sum\limits_{j = 1}^k L(\mathbf{y}^{(i)}, \pi \left[ \text{RP}^{(j)}[V(F'_{\theta'}(\mathbf{x}^{(i)}), \mathbf{y}^{(i)} )] \right]
     \odot F_\theta(\mathbf{x}^{(i)}))  \\
    & + \lambda \KL[V(F'_{\theta'}(\mathbf{x}^{(i)}), \mathbf{y}^{(i)} ) ||H(F'_{\theta'}(\mathbf{x}^{(i)}))],  
    \end{aligned}
\end{equation}
where $\theta' \subset \theta$ denotes the parameter of the feature extractor.
From an optimization perspective, we introduce a tuning hyperparameter to select values that result in a more effective latent representation with minimal impact on the learning process, as discussed in $\beta$-VAE \citep{higgins2016beta}. The hyperparameter $\lambda$ can also balance fitting the potentially noisy label and adhering to the information extracted from the features.
% Here, we add a coefficient $\lambda$ to weight the KL term as beta-VAE~\citep{higgins2016beta}. 
%The KL term can be considered as a regularizer to the meta-network, which is proved to improve the stability of the meta-learning process as indicated in~\citep{bao2021stability}. 

The Monte Carlo estimation strategy for the predictive distribution ensures an efficient feed-forward propagation phase of the model during training. We further analyze the effect of the sampling number in the experimental section. 

For the meta objective \textit{w.r.t.} $V_\phi$, the performance of the meta-network is evaluated on the meta-data set $\mathcal{D}_M$. Since the feed-forward propagation in Eq. (\ref{eq:obj1}) involves the support of $V_\phi$ and $H_\omega$, we denote the updated $\theta$ as $\theta^*(\phi, \omega)$. Therefore, the objective for the meta-network with meta-data $(\tilde{\mathbf{x}}^{(i)},\tilde{\mathbf{y}}^{(i)})$ can be written as
\begin{equation}
\label{eq:obj2}
 \begin{aligned}
\mathop{\argmin}_{\phi, \omega} & \, \mathcal{L}^{meta}(\phi, \omega) = \\ & \frac{1}{M}\sum\limits_{i = 1}^M  L(\tilde{\mathbf{y}}^{(i)}, F(\tilde{\mathbf{x}}^{(i)};\theta^*(\phi, \omega))).
    \end{aligned}
\end{equation}

By minimizing Eq. (\ref{eq:obj2}) \textit{w.r.t.} $\phi$ involved in the updated $F_{\theta^*}$, the learned $V_{\phi^*}$ can achieve unbiased estimation for the posterior and generate rectifying vectors with high fidelity to guide following updates of $\theta$. Also, the prior network $H_{\omega^*}$ restricts $V_{\phi^*}$ to avoid collapsing to produce Dirac delta functions.

\subsubsection{Bi-level optimization} 
We build an iterative optimization algorithm within the bi-level programming framework~\citep{franceschi18a} to obtain the optimal parameters $\{\theta^*, \phi^*, \omega^*\}$ as follows
\begin{equation}\label{eq:bilevel}
\begin{aligned}
\small
\phi^*&, \omega^* =  \mathop{\argmin}_{\phi,\omega}  \mathcal{L}^{meta}( \theta^*(\phi,\omega, \mathcal{D}_N), \mathcal{D}_M), \,  \\  \text{s.t.}, \,  \theta^* &(\phi,\omega, \mathcal{D}_N) = \mathop{\argmin}_{\theta} \mathcal{L}^{emp}(\phi, \omega ,\theta, \mathcal{D}_N).
\end{aligned}
\end{equation}

We adopt stochastic gradient descent (SGD) to solve (\ref{eq:bilevel}). Since the prediction from $F_\theta$ is rectified by $V_\phi$, the gradient for $\theta$ is closely related to $\phi$ and $\omega$. Thus, $\hat{\theta}(\phi,\omega)$ denotes that the updated $\hat{\theta}$ is the function of $\phi$ and $\omega$. Here, we assign a learning rate of $\alpha$. By sampling a mini-batch of $n$ training examples $\{(\mathbf{x}^i,\mathbf{y}^i) \}^n_{i=1}$, the updating step of the classification network $F_\theta$ \textit{w.r.t.} Eq. (\ref{eq:obj1}) can be written as
\begin{equation}\label{eq:theta_step}
\small
    \begin{aligned}
    &\hat{\theta}^{(t)}(\phi, \omega)  = \theta^{(t)} - \alpha  \nabla_\theta \widetilde{\mathcal{L}}^{emp}(\theta), \quad \\ & \text{where}\quad \widetilde{\mathcal{L}}^{emp}(\theta)  =   \frac{1}{k n}\sum\limits_{i = 1}^n \sum\limits_{j = 1}^k  L(\mathbf{y}^{(i)}, \\ & \qquad \pi \left[\text{RP}^{(j)}[V(F'_{\theta'^{(t)}}(\mathbf{x}^{(i)}), \mathbf{y}^{(i)}; \phi )]\right] 
    \odot F_{\theta^{(t)}}(\mathbf{x}^{(i)})) \\ & \quad+ \lambda \KL[V(F'_{\theta'^{(t)}}(\mathbf{x}^{(i)}), \mathbf{y}^{(i)} ; \phi) || H(F'_{{\theta'}^{(t)}}(\mathbf{x}^{(i)}; \omega)].  
   \end{aligned}
\end{equation}

Given a mini-batch of $m$ meta samples $\{(\tilde{\mathbf{x}}^i,\tilde{\mathbf{y}}^i) \}^m_{i=1}$, the learning of $\phi$ and $\omega$ can be achieved by back-propagating through the learning process of $\theta$. Specifically, after obtaining $\hat{\theta}^{(t)}(\phi, \omega)$ with fixed $\phi$ and $\omega$ in Eq. (\ref{eq:theta_step}), the parameter of $\phi$ in the meta-network $V_\phi(\cdot)$ can be updated \textit{w.r.t.} the objective in Eq. (\ref{eq:obj2})
\begin{equation}\label{eq:phi_update}
{\phi^{(t + 1)}} = {\phi^{(t)}} - \eta \frac{1}{m}\sum\limits_{i = 1}^m {{\nabla _\phi }} L(\tilde{\mathbf{y}}^{(i)}, F(\tilde{\mathbf{x}}^{(i)};\hat{\theta}^{(t)}(\phi,\omega ))) ,
\end{equation}
where $\eta$ denotes the step size. Similar update steps for the prior network can be written as
\begin{equation}\label{eq:omega_update}
{\omega^{(t + 1)}} = {\omega^{(t)}} - \eta \frac{1}{m}\sum\limits_{i = 1}^m {{\nabla _\omega }} L(\tilde{\mathbf{y}}^{(i)}, F(\tilde{\mathbf{x}}^{(i)};\hat{\theta}^{(t)}(\phi,\omega )))
\vspace{0mm}
\end{equation}
This bi-level programming manner results in the best hypothesis on the meta-data set, whose theoretical guarantee has been rigorously studied in~\citep{bao2021stability}.

Once $V_\phi$ has been updated, we utilize the current training batch to conduct robustly learning of the classification network $F_{\theta^{(t)}}$ 
\begin{equation}\label{eq:theta_update}
  \begin{aligned}
  \small
    \theta^{(t+1)} & = \theta^{(t)}  -  \alpha \frac{1}{k n}\sum\limits_{i = 1}^n \sum\limits_{j = 1}^k \nabla_\theta  L(\mathbf{y}^{(i)}, F(\mathbf{x}^{(i)};\theta^{(t)})  \\ &\odot \pi \left[\text{RP}^{(j)}[V(F'(\mathbf{x}^{(i)};\theta^{'(t)}), \mathbf{y}^{(i)}; \phi^{(t+1)} )]\right]  ).
  \end{aligned}
\end{equation}

We summarize the overall updating steps in Algorithm \ref{alg:1} and illustrate the main information flow in Figure \ref{fig:fig3}. Estimating the conditional predictive distribution can be efficiently implemented via the Monte Carlo sampling of averaging $k$ results. Indeed, by introducing the variational term, VRI merely requires a small number (\textit{e.g.}, $k=1 \; \text{or} \; 2$) of samples for good performance. By applying the RP trick, the sampling operation is tractable for gradient computation. Therefore, all gradients, including the bi-level programming process, can be efficiently calculated by prevailing differentiation tools.   

\textbf{Complexity}. We analyze this aspect separately for the training and testing phases. During the testing phase, only the classifier network is utilized while the meta-network is put aside, thus incurring no additional computation cost.

During the training phase, the number of samples can potentially increase computation costs, as performance may benefit from a larger sample size. However, due to the inclusion of the variational term in VRI, we achieve higher accuracy than the MC approximation while maintaining efficiency with only one sample. Additionally, the cost of the additional KL loss is negligible since its explicit form can be written as:
$\text{KL}(q, p) = \log \frac{\sigma_2}{\sigma_1} + \frac{\sigma_1^2 + (\mu_1 - \mu_2)^2}{2 \sigma_2^2} - \frac{1}{2}.$ Therefore, compared to deterministic meta-learning methods, our VRI does not incur extra training costs.

In comparison to non-meta-learning methods, VRI requires the computation of the second derivative of the meta-network. Fortunately, the meta-network typically consists of only three fully-connected layers, resulting in a manageable computational cost.

\begin{algorithm}[t]
% \small
\caption{\small The Bi-level optimization for VRI}
\label{alg:1}
\begin{algorithmic}[1]
\REQUIRE Training set $\mathcal{D}_N$, meta set $\mathcal{D}_M$,  batch \\ \quad\quad\quad\, size $n, m$, outer iterations $T$, step size \\ \quad\quad\quad\, $\alpha$, $\eta$, sampling number $k$, \\
\ENSURE Optimal $\theta^*$ \\
\STATE  Initialize parameters $\theta^{(0)}$, $\phi^{(0)}$, and $\omega^{(0)}$ \\
\FOR{$t \in \{1,\dots,T\}$}
        \STATE $ \text{SampleBatch}(\mathcal{D}_N, n), \text{SampleBatch}(\mathcal{D}_M, m)$\\
        \STATE Form learning process of $\hat{\theta}^{(t)}(\phi, \omega)$  \hfill $\triangleright$ Eq. (\ref{eq:theta_step}) \\
        \STATE Optimize $\phi^{(t)}$ with $\hat{\theta}^{(t)}(\phi)$ \hfill $\triangleright$ Eq. (\ref{eq:phi_update}) \\ 
        \STATE Optimize $\omega^{(t)}$ with $\hat{\theta}^{(t)}(\omega)$ \hfill $\triangleright$ Eq. (\ref{eq:omega_update}) \\ 
        \STATE Optimize $\theta^{(t)}$ using the updated $\phi^{(t+1)}$ \\ \hfill $\triangleright$ Eq. (\ref{eq:theta_update}) \\ 
\ENDFOR 
\end{algorithmic}
\end{algorithm}

\subsection{Convergence Analysis}

 The convergence of our proposed Algorithm \ref{alg:1} can be rigorously theoretically guaranteed. Since the meta-network $V(\phi)$ is crucial in our framework, we prove that the algorithm for $V(\phi)$ can converge to the stationary point of the meta loss function under some mild conditions. To facilitate the proof, we adopt the stochastic gradient $\nabla \widetilde{\mathcal{L}}^{meta}\left(\hat{\theta}^{(t)}\left(\phi^{(t)}\right)\right)$ in the following, which is identical to uniformly drawing a mini-batch of samples at random in Eq. (\ref{eq:theta_step}). 
 
\vspace{2mm}
\begin{lemma}[Smoothness]
	\label{lemma1}
 	Suppose the loss function $L$ \textit{w.r.t.} $\theta$ in Eq. (\ref{eq:obj2}) is $\ell$-smooth and $\tau$-Lipschitz, the KL term $\KL$ \textit{w.r.t.} the output of $V(\phi)$ has the $o$-bounded gradient, and $V(\phi)$ is differential with the $\delta$-bounded gradient and twice differential with its $\zeta$-bounded Hessian. Then the meta loss function \textit{w.r.t.} $\theta$ is $\hat{\ell}$-smooth.
\end{lemma}
\begin{proof}
	See Appendix \ref{app:a}.
\end{proof}

For most modern architectures (e.g., CNN, MLP, ReLU, Leaky ReLU, SoftPlus, Tanh, Sigmoid, ArcTan, Softsign, and max-pooling), and the cross-entropy loss, the Lipschitz constant can be computed or estimated \citep{virmaux2018lipschitz}. Therefore, the classification network and Meta-network adhere to the Lipschitz continuity assumption. Regarding the smoothness assumption, the architectures and functions used in our method exhibit local smoothness around the stationary point.

Lemma \ref{lemma1} implies that the meta loss \textit{w.r.t.} the meta-network is smooth-bounded. We provide the convergence rate in Theorem \ref{th1} with the support of this essential property.

\vspace{1mm}
\begin{theorem}[Convergence Rate]
		\label{th1}
		Assume that the variance of the stochastic gradient $ \nabla \widetilde{\mathcal{L}}^{meta}\left(\hat{\theta}^{(t)}\left(\phi^{(t)}\right)\right)$ is bounded $\mathbb{E}\left[\left\|\nabla \widetilde{\mathcal{L}}^{meta}\left(\hat{\theta}^{(t)}\left(\phi^{(t)}\right)\right)- \nabla\mathcal{L}^{meta}\left(\hat{\theta}^{(t)}\left(\phi^{(t)}\right)\right) \right\|_2^2\right]$ $ \le \sigma^2 < \infty$. Following directly from Lemma \ref{lemma1},  let the learning rate $\alpha_t$ satisfies $\alpha_t=\min\{1,\frac{\kappa}{T}\}$, for some $\kappa>0$, such that $\frac{\kappa}{T}<1$, and $\eta_t, 1\leq t\leq T$ is a monotone descent sequence, $\eta_t =\min\{\frac{1}{\hat{\ell}},\frac{C}{\sigma\sqrt{T}}\} $ for some $C>0$, such that $\frac{\sigma\sqrt{T}}{C}\geq \hat{\ell}$ and $\sum_{t=1}^\infty \eta_t \leq \infty,\sum_{t=1}^\infty \eta_t^2 \leq \infty $. Then we have 
		\begin{equation}
		    \begin{aligned}
			\frac{1}{T}\sum\limits_{t= 1}^T \mathbb{E}\left[ \left\| \nabla\mathcal{L}^{meta}\left(\hat{\theta}^{(t)}(\phi^{(t)})\right) \right\|_2^2\right] \leq \mathcal{O}(\frac{1}{\sqrt{T}}).
		\end{aligned}
		\end{equation}
	\end{theorem}
\begin{proof}
See Appendix \ref{app:a}.
\end{proof}

%More specifically, Theorem \ref{th1} implies that our learning algorithm VRI can achieve $\mathbb{E}\left[ \left\| \nabla\mathcal{L}^{meta}\left(\hat{\theta}^{(t)}\left(\phi^{(t)}\right)\right) \right\|_2^2\right] \leq \epsilon$ in $\mathcal{O}(1/\epsilon^2)$ steps. As the iteration step increases, the algorithm would ultimately converge to a stationary point.

More specifically, Theorem \ref{th1} implies that our learning algorithm VRI can find an $\epsilon$-first-order stationary point for any small positive $\epsilon$ and achieve $\mathbb{E}\left[ \left\| \nabla\mathcal{L}^{meta}\left(\hat{\theta}^{(t)}\left(\phi^{(t)}\right)\right) \right\|_2^2\right] \leq \epsilon$ in $\mathcal{O}(1/\epsilon^2)$ steps. This is identical to the typical deterministic bi-level algorithm for meta-learning, such as MAML \citep{fallah2020convergence}. As the iteration step increases, the algorithm would ultimately converge to a stationary point.

\begin{table*}[t]
	\caption{Architectures of applied meta-net in experiments.}
 \hspace{-4mm}
	\centering
    \small
	\label{tab:network}
	\scalebox{0.85}{
 \setlength{\tabcolsep}{3mm}{
	\begin{tabular}{c|cl}
		\toprule[1.1pt]
	&	\textbf{Output size} & \textbf{Layers} \\
		\midrule
	&	$1024$ & Input ConCat (sample features, embedding labels) \\
\multicolumn{1}{l|}{\textbf{$V_\phi(\cdot)$}}	&$512$ & fully connected, tanh \\
	&	Number of classes & fully connected, Sigmoid to $\bm{\mu}_\mathbf{v}$, $\log\bm{\sigma}^2_\mathbf{v} $ \\ \midrule
	&	$1024$ & Input sample features \\
    \multicolumn{1}{l|}{\textbf{$H_\omega(\cdot)$}} & $512$ & fully connected, tanh \\
		& Number of classes & fully connected, Sigmoid to $\bm{\mu}_\mathbf{v}$, $\log\bm{\sigma}^2_\mathbf{v} $ \\
		\bottomrule[1.1pt]
	\end{tabular}
 }	}
\end{table*}

\begin{table*}[t]
	\caption{Hyperparameters of the classification network in our experiments on different datasets.}
 \hspace{-4mm}
	\centering
    \small
	\label{tab:params}
	\scalebox{0.85}{
 \setlength{\tabcolsep}{3mm}{
	    \begin{tabular}{l|c|c|c|c|c}
		%\toprule
		\toprule[1.1pt]
		%\footnotesize{\diagbox{Hyperparams}{Dataset}} 
		\textbf{Dataset} & CIFAR-10 & CIFAR-100 & Clothing1M &Food-101N & ANIMAL-10N  \\
  \midrule
		Sampling Number       & 2        & 2        & 1          & 1   & 1      \\
		Batch Size          & 100      & 100       & 128        & 128     & 128   \\
		Optimizer           & SGD      & SGD       & SGD        & Adam   & Adam  \\
		Initial Learning Rate          & 0.02      & 0.02       & 0.02        & 3e-4   & 3e-4  \\
		Decay Rate          & 5e-4     & 5e-4      & 5e-4       & -      & -  \\
		Total Epoch Number              & 160       & 160        & 10         & 30     & 30    \\
		Momentum            & 0.9      & 0.9       & 0.9        & -      & -   \\
		\bottomrule[1.1pt]
	\end{tabular}}	}
\end{table*}

\section{Experiments}\label{sec:experiments}

We conduct classification experiments with variant noise types on five benchmarks, including three real-world datasets, and obtain better performance compared with the state-of-the-art (SOTA) method. The exhaustive analysis further demonstrates the virtue of the proposed model on the task of learning with noisy labels. The code is available at {\url{https://github.com/haolsun/VRI}}.

\subsection{Setup} 

\noindent\textbf{Datasets}. We evaluate VRI on five benchmarks of CIFAR-10, CIFAR-100, Clothing1M~\citep{xiao2015learning}, Food-101N~\citep{lee2018cleannet}, and ANIMAL-10N \citep{song2019selfie}, and follow the consistent experimental protocol in~\citep{shu2019meta, zhangyikai2021learning} for the fair comparison. We randomly select 1000 training samples (2\%) as metadata for CIFAR-10 \& 100. For Clothing1M and Food-101N, we use the validation set for meta-learning. More details for constructing those datasets are provided as follows.

\textbf{CIFAR-10}~\citep{krizhevsky2009learning} dataset consists of 60,000 images of 10 categories. We adopt the splitting strategy in~\citep{shu2019meta} by randomly selecting 1,000 samples from the training set to construct the meta dataset. We train the classification network on the remaining 40,000 noisy samples and evaluate the model on 1,0000 testing images. 

\textbf{CIFAR-100}~\citep{krizhevsky2009learning} is more challenging than CIFAR-10 including 100 classes belonging to 20 superclasses where each category contains 600 images with the resolution of 32 $\times$ 32. Similar splitting manners as CIFAR-10 are employed. 

\textbf{Clothing1M}~\citep{xiao2015learning} is a large-scale dataset that is collected from real-world online shopping websites. It contains 1 million images of 14 categories whose labels are generated based on tags extracted from the
surrounding texts and keywords, causing huge label noise. The estimated percentage of corrupted 
labels is around 38.46\%. A portion of clean data is also included in Clothing1M, which has been divided into the training set (903k images), validation set (14k images), and test set (10k images). We select the validation set as the meta-dataset and evaluate the performance of the test set. We resize all images to $256 \times 256$ as in~\citep{shu2019meta}.

\begin{table*}[htbp]
\centering
\small
\caption{Testing Accuracy (\%) on CIFAR-10 and CIFAR-100 with \textbf{Flip} label noise. The best and second-best performances are highlighted with \textbf{blod} and \underline{underline}, respectively.}
\label{tab:flip}
\scalebox{0.88}{
\setlength{\tabcolsep}{2mm}{
\begin{tabular}{lc|c|cc|cc}
\toprule[1.3pt]
\multicolumn{2}{l|}{{Dataset}} & & \multicolumn{2}{c|}{CIFAR-10} & \multicolumn{2}{c}{CIFAR-100} \\ \midrule
\multicolumn{2}{l|}{{Noise Ratio}} & Structure & 20\% & 40\%  & 20\% & 40\%       \\ \midrule
{Baseline}                                      & &ResNet-32   & 76.83 $\pm$ 0.3     & 70.77 $\pm$ 2.3   & 50.86 $\pm$ 0.3      & 43.01 $\pm$ 1.2  \\
MW-Net {\footnotesize\citep{shu2019meta}} & {\footnotesize(NeurIPS19)} &ResNet-32    & 90.33 $\pm$ 0.6  & 87.54 $\pm$ 0.2    & 64.22 $\pm$ 0.3  & 58.64 $\pm$ 0.5   \\ 
MLC {\footnotesize~\citep{wang2020training}} & {\footnotesize(CVPR20)} &ResNet-32      & 90.07 $\pm$ 0.2   & 88.97 $\pm$ 0.5    & 64.91 $\pm$ 0.4  & 59.96 $\pm$ 0.6  \\
CORES* {\footnotesize~\citep{cheng2021learning}} & {\footnotesize(ICLR21)} &ResNet-32 & \underline{91.41 $\pm$ 0.4}  & 89.47 $\pm$ 0.3 & 64.82 $\pm$ 0.5  & \underline{62.76 $\pm$ 0.4} \\
PMW-Net {\footnotesize~\citep{zhao2021probabilistic}} & {\footnotesize(TNNLS23)}&ResNet-32 & 90.47 $\pm$ 0.1  & 87.69 $\pm$ 0.3   & 64.95 $\pm$ 0.2  & 58.72 $\pm$ 0.2\\
WarPI {\footnotesize~\citep{sun2021learning}} & {\footnotesize(PR22)} &ResNet-32 & 90.93  & \underline{89.87}  & 65.52  & 62.37 \\
FaMUS {\footnotesize~\citep{Xu2021FaMUS}} & {\footnotesize(CVPR21)} &ResNet-32 & 90.78   & 88.91  & \underline{65.79}    & 59.66 \\
FSR {\footnotesize~\citep{zhang2021learning}} & {\footnotesize(ICCV21)} & ResNet-32 
& 91.50 & 90.20  & 68.59 & 66.03 \\
\textbf{VRI} {\footnotesize(Ours)} &  &ResNet-32   & \textbf{91.93 $\pm$ 0.1}  & \textbf{91.21 $\pm$ 0.3}   & \textbf{66.03 $\pm$ 0.2} & \textbf{65.04 $\pm$ 0.4}  \\ \midrule
{DivideMix} {\footnotesize~\citep{li2020dividemix}} & {\footnotesize(ICLR20)}  &ResNet-18            & - & 93.4   &-& 72.1   \\
ELR {\footnotesize~\citep{liu2020early}} & {\footnotesize(NeurIPS20)} &ResNet-34  & 93.28 $\pm$ 0.2   & 90.35 $\pm$ 0.4    & \underline{74.20 $\pm$ 0.3}    & \textbf{73.73 $\pm$ 0.3}  \\
JNPL {\footnotesize~\citep{kim2021joint}}  &{\footnotesize(CVPR21)}  &ResNet-34  & \underline{93.45}    & 90.72      & 69.95     & 59.51 \\
SR {\footnotesize~\citep{zhou2021learning}} & {\footnotesize(ICCV21)} &ResNet-34  & 89.55 $\pm$ 0.3  & 85.45 $\pm$ 0.2  & 64.79 $\pm$ 0.1   & 49.51 $\pm$ 0.6  \\
MSLC {\footnotesize~\citep{yichen2020softlabel}} & {\footnotesize(AAAI21)} &ResNet-34  & 94.11  & \underline{92.48}  & 70.20  & 69.24 \\
SFT  {\footnotesize\citep{wei2022self}} &  {\footnotesize(ECCV22)} & ResNet-34 &91.53 $\pm$ 0.3  &89.93 $\pm$ 0.5   & 71.23 $\pm$ 0.3   & \underline{69.29 $\pm$ 0.4} \\
GSS-SSL {\footnotesize\citep{yu2023prevent}} & {\footnotesize(CVPR23)} & ResNet-34 & 93.42 $\pm$ 0.1  & 91.82 $\pm$ 0.1   & 73.81 $\pm$ 0.2  & 65.84 $\pm$ 0.2 \\
FasTEN {\footnotesize~\citep{kye2022learning}} & {\footnotesize(ECCV22)} & ResNet-34 
& 92.29 & 90.43 & 70.25 & 67.93 \\
EMLC {\footnotesize~\citep{taraday2023enhanced}} & {\footnotesize(ICCV23)} & ResNet-34
& 91.50    & 89.84   & 70.05      & 60.89 \\
\textbf{VRI} {\footnotesize(Ours)}  &  &ResNet-34    & \textbf{93.79 $\pm$ 0.1}   & \textbf{93.27 $\pm$ 0.2}  & \textbf{75.13 $\pm$ 0.2}   & 67.81 $\pm$ 0.3   \\
\bottomrule[1.3pt]
\end{tabular}
}
}
\end{table*}

\begin{table*}[htbp]
\centering
\small
\caption{Testing Accuracy (\%) on CIFAR-10 and CIFAR-100 with \textbf{Uniform} label noise.  The best and second-best performances are highlighted with \textbf{blod} and \underline{underline}, respectively.}
\label{tab:uniform}
\scalebox{0.9}{
\setlength{\tabcolsep}{2mm}{
\begin{tabular}{lc|c|cc|cc}
\toprule[1.3pt]
\multicolumn{2}{l|}{{Dataset}} & & \multicolumn{2}{c|}{CIFAR-10} & \multicolumn{2}{c}{CIFAR-100} \\ \midrule
\multicolumn{2}{l|}{{Noise Ratio}} & Structure & 20\% & 40\%  & 20\% & 40\%       \\ \midrule
ELR {\footnotesize~\citep{liu2020early}} & {\footnotesize(NeurIPS20)} & ResNet-34    & 91.43 $\pm$ 0.2      & 88.87 $\pm$ 0.2   & 68.43 $\pm$ 0.4   & 60.05 $\pm$ 0.9  \\ 
MSLC {\footnotesize~\citep{yichen2020softlabel}} & {\footnotesize(AAAI21)} & ResNet-34   & 91.42  & 87.25  & 68.70  & 60.25  \\
FaMUS {\footnotesize~\citep{Xu2021FaMUS} } & {\footnotesize(CVPR21)} & ResNet-18   & 90.50  & 85.80  & 69.40  & \textbf{62.90}  \\ 
SFT {\footnotesize~\citep{wei2022self}} & {\footnotesize(ECCV22)} & ResNet-18  & 89.54 $\pm$ 0.3 & - & \underline{69.72 $\pm$ 0.3} & - \\
SOP{\footnotesize ~\citep{liu2022robust}}   & {\footnotesize(ICML22)} & ResNet-34 & 90.09 $\pm$ 0.3 & 86.78 $\pm$ 0.2 & \textbf{70.12 $\pm$ 0.5} & \underline{60.06 $\pm$ 0.4} \\
FSR {\footnotesize~\citep{zhang2021learning}} & {\footnotesize(ICCV21)} & ResNet-32 
& \underline{91.84} & \textbf{90.20} & 65.78 & 62.79 \\
FasTEN {\footnotesize~\citep{kye2022learning}} & {\footnotesize(ECCV22)} & ResNet-34 
& \textbf{91.94} & \underline{90.07} & 68.75 & \textbf{63.82} \\
EMLC {\footnotesize~\citep{taraday2023enhanced}} & {\footnotesize(ICCV23)} & ResNet-34
&  91.06    & 88.54     & - & - \\
\textbf{VRI} {\footnotesize(Ours)} & &ResNet-18        & 91.34 $\pm$ 0.2    & 87.68 $\pm$ 0.3       & {68.92 $\pm$ 0.2}       &  {62.12 $\pm$ 0.2} \\ 
% \midrule
% {Baseline}        &  &WResNet-28-10                & 68.07 $\pm$ 1.2   & 53.12 $\pm$ 3.0    & 51.11 $ \pm $ 0.4    & 30.92 $\pm$ 0.3   \\
% {MentorNet} {\footnotesize~\citep{jiang2018Mentornet}} & {\footnotesize(ICML18)} &WResNet-28-10     & 87.33 $\pm$ 0.2    & 82.80 $\pm$ 1.4    & 61.39 $\pm$ 4.0     & 36.87 $\pm$ 1.5   \\
% {MW-Net} {\footnotesize~\citep{shu2019meta}} & {\footnotesize(NeurIPS19)} &WResNet-28-10            & 89.27 $\pm$ 0.3    & 84.07 $\pm$ 0.3    & {67.73 $\pm$ 0.3}   & {58.75 $\pm$ 0.1}   \\ 
% {MLC} {\footnotesize~\citep{wang2020training}} & {\footnotesize(CVPR20)} &WResNet-28-10             & 89.20 $\pm$ 0.1    & 84.22 $\pm$ 0.3    & -   & - \\ 
% {DMI-NS} {\footnotesize~\citep{chen2021robustness}} & {\footnotesize(AAAI21)} &WResNet-28-10  & \underline{91.11 $\pm$ 0.5}  & 83.46 $\pm$ 0.5    & 66.95 $\pm$ 0.2     & 58.35 $\pm$ 0.1  \\
% WarPI {\footnotesize~\citep{sun2021learning}} & {\footnotesize(PR22)} & WResNet-28-10 & 89.73 & \underline{84.44}  & \underline{67.90}   & \underline{59.04} \\
% \textbf{VRI} {\footnotesize(Ours)} &  & WResNet-28-10     & \textbf{91.29 $\pm$ 0.2}       & \textbf{84.68 $\pm$ 0.2}      & \textbf{67.92 $\pm$ 0.2}       & \textbf{59.32 $\pm$ 0.3} \\
\bottomrule[1.3pt]
\end{tabular}
}
}
\end{table*}

\begin{table*}[htbp]
\centering
\small
\caption{Testing Accuracy (\%) on CIFAR-10 and CIFAR-100 with \textbf{Instance-dependent} label noise. Note that ``ResNet-18/34" denotes applying ResNet-18 for CIFAR-10 and  ResNet-34 for CIFAR-100. The best and second-best performances are highlighted with \textbf{blod} and \underline{underline}, respectively.}
\label{tab:instance}
\scalebox{0.85}{
\setlength{\tabcolsep}{2mm}{
\begin{tabular}{lc|c|cc|cc}
\toprule[1.3pt]
\multicolumn{2}{l|}{{Dataset}} & & \multicolumn{2}{c|}{CIFAR-10} & \multicolumn{2}{c}{CIFAR-100} \\ \midrule
\multicolumn{2}{l|}{{Noise Ratio}} & Structure & 20\% & 40\%  & 20\% & 40\%       \\ \midrule
{Baseline}                    &                   & ResNet-18 / 34     & 85.10 $\pm$ 0.6            & 77.00 $\pm$ 2.1            & 52.19 $\pm$ 1.4    & 42.26 $\pm$ 1.2   \\
{Co-teaching} {\footnotesize~\citep{han2018co}}  & {\footnotesize(NeurIPS18)} & ResNet-18 / 34      & 86.54 $\pm$ 0.1            & 79.98 $\pm$ 0.3            & 57.24 $\pm$ 0.6    & 45.69 $\pm$ 0.9   \\
{Peer loss}  {\footnotesize\citep{liu2020peer}} & {\footnotesize(ICML20)} & ResNet-18 / 34      & 88.19 $\pm$ 0.5            & 81.53 $\pm$ 0.7             & 63.82 $\pm$ 0.3   & 47.91 $\pm$ 0.5 \\
{CORES*}  {\footnotesize\citep{cheng2021learning}} & {\footnotesize(ICLR21)} & ResNet-18 / 34   & 89.67 $\pm$ 0.3            & 82.99 $\pm$ 0.5           & 64.86 $\pm$ 0.5    & 49.62 $\pm$ 0.7 \\
WarPI {\footnotesize~\citep{sun2021learning}} & {\footnotesize(PR22)} &ResNet-18 / 34 & 89.76 $\pm$ 0.4 & 87.57 $\pm$ 0.9 & 65.08 $\pm$ 0.6 & 57.38  $\pm$ 1.0 \\
{CDR} {\footnotesize\citep{xia2020robust}}   & {\footnotesize(ICLR21)}    & ResNet-18 / 34       & 90.41 $\pm$ 0.3           & 83.07 $\pm$ 1.3            & 67.33 $\pm$ 0.6    & 55.94 $\pm$ 0.5 \\
{Me-Momen.} {\footnotesize\citep{bai2021me}}  & {\footnotesize(ICCV21)}    & ResNet-18 / 34   & 90.86 $\pm$ 0.2           & 86.66 $\pm$ 0.9            & 68.11 $\pm$ 0.5    & 58.58 $\pm$ 1.2 \\
FaMUS  {\footnotesize\citep{Xu2021FaMUS}} & {\footnotesize(CVPR21)} & ResNet-18 / 34    & 91.23 $\pm$ 0.3  & 89.88 $\pm$ 0.6  & 66.65 $\pm$ 0.5   & 57.21 $\pm$ 1.2 \\
{PES} {\footnotesize\citep{bai2021understanding}} & {\footnotesize(NeurIPS21)}  & ResNet-18 / 34   & \underline{92.69 $\pm$ 0.4}          & 89.73 $\pm$ 0.5             & 70.49 $\pm$ 0.7  & 65.68 $\pm$ 1.4  \\ 
% SFT {\footnotesize\citep{wei2022self}} & {\footnotesize(ECCV22)}& ResNet-18 / 34  &91.41 $\pm$ 0.3 & 89.97 $\pm$ 0.5 & \textbf{71.83 $\pm$ 0.4} & \textbf{69.91 $\pm$ 0.5} \\
Late Stop {\footnotesize\citep{yuan2023late}} & {\footnotesize(ICCV23)} & ResNet-18 / 34 & 91.08 $\pm$ 0.2 &  87.41 $\pm$ 0.4 &  68.59 $\pm$ 0.7  & 59.28 $\pm$ 0.5 \\ 
PADDLES {\footnotesize\citep{huang2023paddles}} & {\footnotesize(ICCV23)} & ResNet-18 / 34 & \textbf{92.76 $\pm$ 0.3}  & \underline{89.87 $\pm$ 0.5}  & \underline{70.88 $\pm$ 0.6}         & \underline{66.11 $\pm$ 1.2} \\ 
EMLC {\footnotesize~\citep{taraday2023enhanced}} & {\footnotesize(ICCV23)} & ResNet-18
& 91.76    & 89.05     & 69.74     & 68.06 \\
\textbf{VRI} {\footnotesize(Ours)}  & & ResNet-18        & 92.13 $\pm$ 0.3  & \textbf{90.60 $\pm$ 0.4}  & \textbf{71.24 $\pm$ 0.2} & \textbf{68.17 $\pm$ 0.5} \\
\bottomrule[1.3pt]
\end{tabular}
}
}
\end{table*}

\textbf{Food-101N}~\citep{lee2018cleannet} is constructed based on the taxonomy of $101$ categories in Food-101~\citep{bossard2014food}. It consists of 310k images collected from Google, Bing, Yelp, and TripAdvisor. The noise ratio for labels is around 20$\%$. We select the validation set of 3824 as the meta-data. Following the testing protocol in~\citep{lee2018cleannet, zhangyikai2021learning}, we learn the model on the training set of 55k images and evaluate it on the testing set of the original Food-101. 

\textbf{ANIMAL-10N}~\citep{song2019selfie} contains human-labeled online images for 5 pairs of animals with confusing appearance. The estimated label noise rate is 8\%. There are 50,000 training and 5,000 testing images with the resolution of 64 $\times$ 64. We evaluate our model on the dataset without a clean meta set.

\begin{table*}[t]
  \centering
  \caption{Testing Accuracy (\%) on \textbf{real-world noise}, including Clothing1M and Food-101N.}
     \label{tab:cloth}
  \small
  \scalebox{0.8}{
  \begin{tabular}{lc|lc|lc|lc}
\toprule[1.1pt]
 \multicolumn{4}{c|}{Clothing1M}  &\multicolumn{4}{c}{Food-101N}\\ \midrule
{MWNet}{\footnotesize ~\citep{shu2019meta}}  & 73.72   
&{DivideMix}{\footnotesize~\citep{li2020dividemix}}  & 74.76
& {Base Model}  & 81.67
&$\text{CNet}_{\text{h}}${\footnotesize~\citep{lee2018cleannet}} & 83.47 \\

 ELR~{\footnotesize\citep{liu2020early}}  & 74.81 
& {CAL}{\footnotesize~\citep{Zhu_2021_CVPR}} & 74.17
& {MWNet}{\footnotesize~\citep{shu2019meta}}    & 84.72
& {SMP}{\footnotesize~\citep{han2019deep}}  & 85.11 \\

{PLC}{\footnotesize~\citep{zhangyikai2021learning}}       & 73.24
& {WarPI}{\footnotesize~\citep{sun2021learning}}       & 74.98
& {NRank}{\footnotesize~\citep{sharma2020noiserank}} & 85.20
&  ELR+ {\footnotesize\citep{liu2020early}}   & 85.77  \\

 JNPL{\footnotesize~\citep{kim2021joint}}          & 74.15
 & CoDis {\footnotesize\citep{xia2023combating}}  & 74.92
& {PLC}{\footnotesize~\citep{zhangyikai2021learning}}      & 85.28 
 &  {WarPI}{\footnotesize~\citep{sun2021learning}}       & 85.91\\

 NCR{\footnotesize~\citep{iscen2022learning}}  & 74.42 
 & \textbf{VRI} {\footnotesize(Ours)}      & \textbf{75.19} 
 & CoDis {\footnotesize\citep{xia2023combating}}  & 86.13
 & \textbf{VRI} {\footnotesize(Ours)}  & \textbf{86.24}  \\
\bottomrule[1.1pt]
\end{tabular}
  }
\end{table*}

\begin{table*}[t]
\caption{Average test accuracy on CIFAR80N, an \textbf{open-set} label noisy dataset, over the last 10 epochs.} 
\small
\centering
\scalebox{0.92}{
\setlength{\tabcolsep}{2mm}{
\begin{tabular}{lc|cccc}
\toprule[1.1pt]
\multicolumn{1}{l}{Methods} &    & \emph{Unif 20\%}    & \emph{Unif 50\%} & \emph{Unif 80\%}     & \emph{Flip 40\%} \\ \midrule
Cross-Entropy               & -              & 42.57   & 27.06  & 9.27      & 22.25 \\
Decoupling {\footnotesize \citep{malach2017decoupling} }                 & {\footnotesize NIPS 2017}      & 43.49   & -      & 10.01     & 33.74 \\
Co-teaching {\footnotesize \citep{han2018co} }                & {\footnotesize NIPS 2018}      & 60.38   & -      & 16.59     & 42.42 \\
Co-teaching+ {\footnotesize \citep{yu2019does}  }             & {\footnotesize  ICML 2019 }     & 53.97   & -      & 12.29     & 43.01 \\
JoCoR {\footnotesize \citep{wei2020combating}}                & {\footnotesize CVPR 2020}      & 59.99   & -      & 12.85     & 39.36 \\
MoPro  {\footnotesize \citep{li2021mopro} }   & {\footnotesize ICLR 2021}      & 65.60   & -      & 30.29     & 60.22 \\
Jo-SRC {\footnotesize \citep{yao2021jo}}     & {\footnotesize CVPR 2021}      & 65.83   & 58.51  & 29.76     & 53.03  \\
PNP-hard {\footnotesize \citep{sun2022pnp}}  & {\footnotesize CVPR 2022}      & 65.87   & -      & 30.79     & 56.17 \\ 
PNP-soft {\footnotesize \citep{sun2022pnp}}  & {\footnotesize CVPR 2022}     & 67.00   &  -     & \underline{34.36}     & \underline{61.23}  \\ 
USDNL {\footnotesize \citep{xu2023usdnl}}    & {\footnotesize AAAI 2023}      & \underline{71.54}   & \underline{63.98}  & 26.07     & -     \\ \midrule
\textbf{VRI}                         &            & \textbf{72.90} & \textbf{67.03} & \textbf{35.44} & \textbf{64.71}  \\ 
\bottomrule[1.1pt]
\end{tabular}
}
}
\label{tab:openset}
\end{table*}

\vspace{1mm}
\noindent\textbf{Noise settings}. We conduct experiments to study four types of corrupted labels. 1) For \textit{flip noise}, we randomly select a transition class for each class and form the label noise by flipping the label to the transition class with a certain probability $\rho$. 2) For \textit{uniform noise}, we independently change the label to a random class with a probability of $\rho$. 3) For \textit{instance-dependent (ID) noise}, we adopt the strategy in \citep{xia2020part} to construct the dataset with noise caused by the uncertain annotation of the ambiguous observation. 4) For \textit{real-world noise}, different from the above synthetic noise, it is introduced at the stage of data collection in the real world with diverse forms of noise. 5) For \textit{openset noise}, we adopt CIFAR-80N, which provides a training set with out-of-distribution samples.
For flip, uniform, and ID noise, we conduct experiments under variant settings of noise ratios on CIFAR-10 \& 100, where $\rho \in \{0.2, 0.4, 0.6\}$. 

\vspace{1mm}
\noindent\textbf{Network architectures}. The architecture of the classification network affects the performance. We present the result with different backbones in the following comparison experiments. 
Following the setting in~\citep{ren2018learning,shu2019meta,zhangyikai2021learning}, we adopt ResNet-18\&32\&34~\citep{resnet}, Wide ResNet-28-10~\citep{zagoruyko2016wide}, and ResNet-50~\citep{resnet} in the following experiments. Note that ResNet-32 is a tiny model which is much slimmer than ResNet-18/34. We implement the meta-network and prior network as the three-layer fully-connected network whose dimension for hidden layers is set as $1024$. Since its inputs are the feature embedding concatenated with the one-hot label vector, the input dimension is $k+c$, where $k, c$ are the dimension of the feature embedding and the number of categories, respectively. Besides, the dimension of the output layer of the meta-network is $2c$. For the meta-network of $V_\phi(\cdot)$ and the prior-network of $H_\omega(\cdot)$, all models share the same architecture, as in Table \ref{tab:network}.

\vspace{1mm}
\noindent\textbf{Other hyperparameters}. The weight coefficient $\lambda$ for the KL term is set to be 0.001 for all experiments. Its sensitivity for the generalization performance is analyzed in the ablation study. The sampling number of $k$ is set as $2$ for CIFAR-10 \& 100 and $1$ for Clothing1M, Food-101N, and ANIMAL-10N. For the prior and meta networks, we select the Adam optimizer and set the learning rate as 3e-4 for all experiments. We adopt the CosineAnnealing strategy for adjusting the learning rate of the classification network on CIFAR-10 \& 100. 
Settings of other hyperparameters for the classification network are listed in Table \ref{tab:params}.

\subsection{Comparison results}
\noindent\textbf{Synthetic Noise}. 
We evaluate the model on two basic benchmark datasets, \textit{i.e.}, CIFAR-10 and CIFAR-100 of classification tasks. We study variant settings of types and ratios of label noise. For flip \& ID noise, we present results with the setting of $20\%$ and $40\%$ noise ratios. For uniform noise, we choose a more challenging setting of $40\%$ and $60\%$ ratios. For a fair comparison, the settings of generating noisy data and network architectures are consistent for all methods. The comparison baseline methods include the Base Model that is directly trained on corrupted data, 
and other prevailing approaches (\textit{e. g.}, 
DivideMix~\citep{li2020dividemix}, ELR~\citep{liu2020early}, MentorNet~\citep{jiang2018Mentornet}, CORES*~\citep{cheng2021learning}, SFT \citep{wei2022self}, GSS-SSL \citep{yu2023prevent}, PES \citep{bai2021understanding}, CoDis \citep{xia2023combating}, SOP \citep{liu2022robust}, Late Stopping \citep{yuan2023late}, PADDLES \citep{huang2023paddles} and Me-Momentum \citep{bai2021me}), and meta-learning methods including MSLC~\citep{yichen2020softlabel}, MW-Net~\citep{shu2019meta}, PMW-Net~\citep{zhao2021probabilistic}, MLC~\citep{wang2020training}, FSR \citep{zhang2021learning}, FasTEN \citep{kye2022learning}, EMLC \citep{taraday2023enhanced} and FaMUS \citep{Xu2021FaMUS}. Note that other works \citep{li2019learning} with fewer fixed transition patterns for flip noise have not been included. To illustrate the effectiveness of the variational form of learning to rectify loss functions, we also compare the method with the homogeneous MC approximation model, WarPI \citep{sun2021learning}.

%As shown in Table \ref{tab:flip}, \ref{tab:uniform}, and \ref{tab:instance}, VRI outperforms SOTA meta-learning methods on the classification task and achieves superior performance under the setting of flip noise with ResNet-32.
%%%% By using boosting techniques, we highlight that VRI$^*$ achieves the best performance on three types of noise on variant ratios.
%We would like to highlight that our method gains significant improvement of $\mathbf{4.56\%}$ on CIFAR-100 with $20\%$ ID noise compared with the SOTA method of PADDLES. Besides, VRI consistently outperforms the homologous method of WarPI, indicating the superiority of our variational modeling. 

As shown in Table \ref{tab:flip}, \ref{tab:uniform}, and \ref{tab:instance}, compared to other advanced meta-learning-based methods, our method, VRI, shows superior performance in addressing Flip and Instance-dependent label noise. Notably, when compared to the state-of-the-art approach, EMLC \citep{taraday2023enhanced}, our VRI method achieves an improvement of 2.29\% on CIFAR-10 and 5.08\% on CIFAR-100 under 20\% Flip label noise. Besides, VRI consistently outperforms the homologous method of WarPI, indicating the superiority of our variational modeling.

\begin{figure*}[t]
\centering
\vspace{-3mm}
\subfigure[The loss distribution of training data]{
\begin{minipage}[b]{0.4\linewidth}
\centering
\includegraphics[width=1\linewidth]{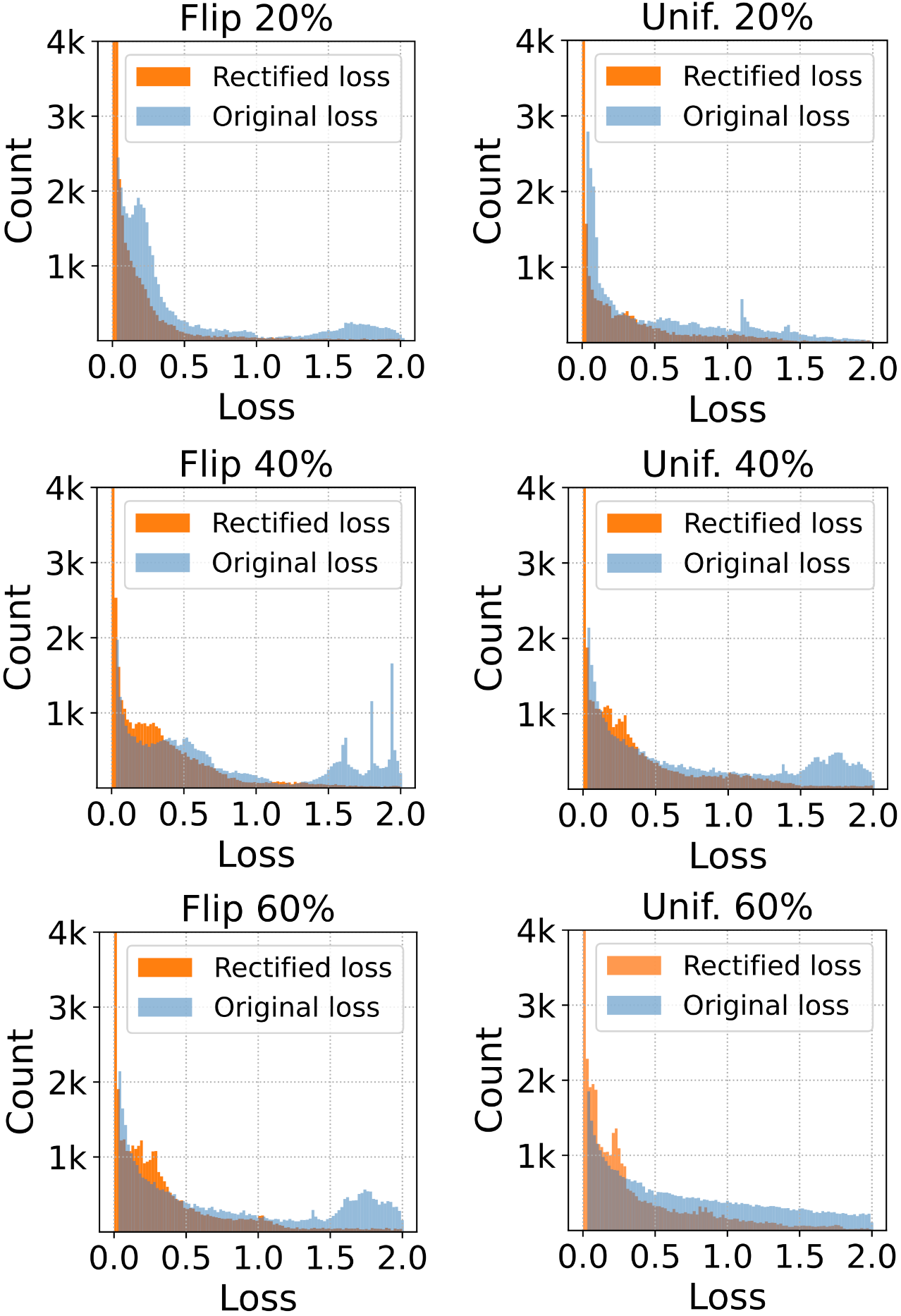}
\end{minipage}}
\hspace{10pt}
\subfigure[Estimation of noise transition matrix]{
\begin{minipage}[b]{0.45\linewidth}
\centering
\includegraphics[width=0.99\linewidth]{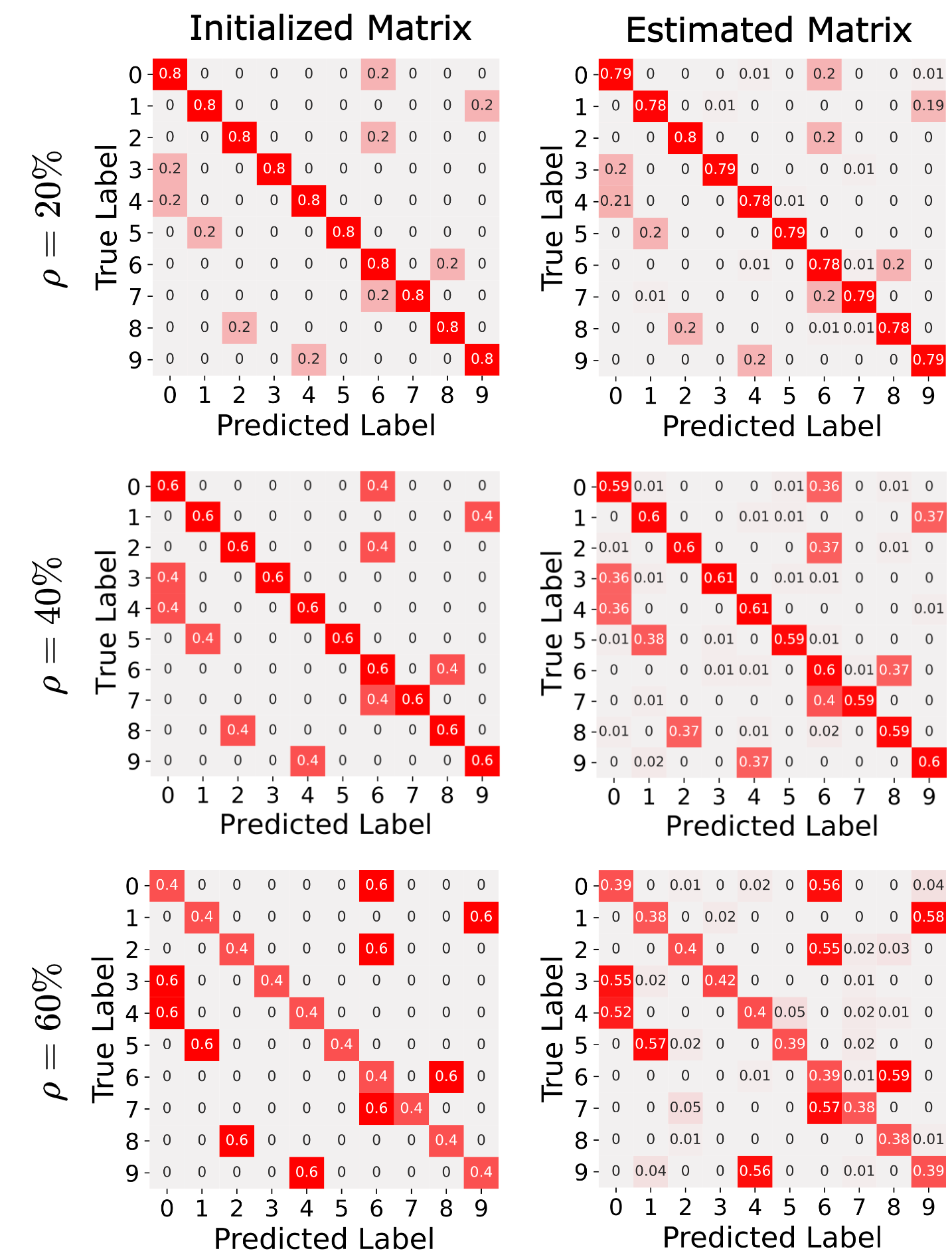}
\end{minipage}}
\caption{(a) As the noise ratio rises, the rectification effect becomes more obvious since the area of the original loss increases. (b) We almost achieve the unbiased estimation for the initialized transition matrix of flip noise with varying noise ratio $\rho$.}
\vspace{-2mm}
\label{fig:vis}
\end{figure*}

% \begin{figure}[]
% \centering
% \includegraphics[width=.85\linewidth]{ collapse.pdf} 
% \caption{Model collapse of MC approximation.}
% \label{fig:collapse}
% \end{figure}

% \begin{figure}[]
% \centering
% \includegraphics[width=.85\linewidth]{ collapse.pdf} 
% \caption{Model collapse of MC approximation.}
% \label{fig:collapse}
% \end{figure}

\begin{figure}[t]
\centering
\includegraphics[width=0.45\linewidth]{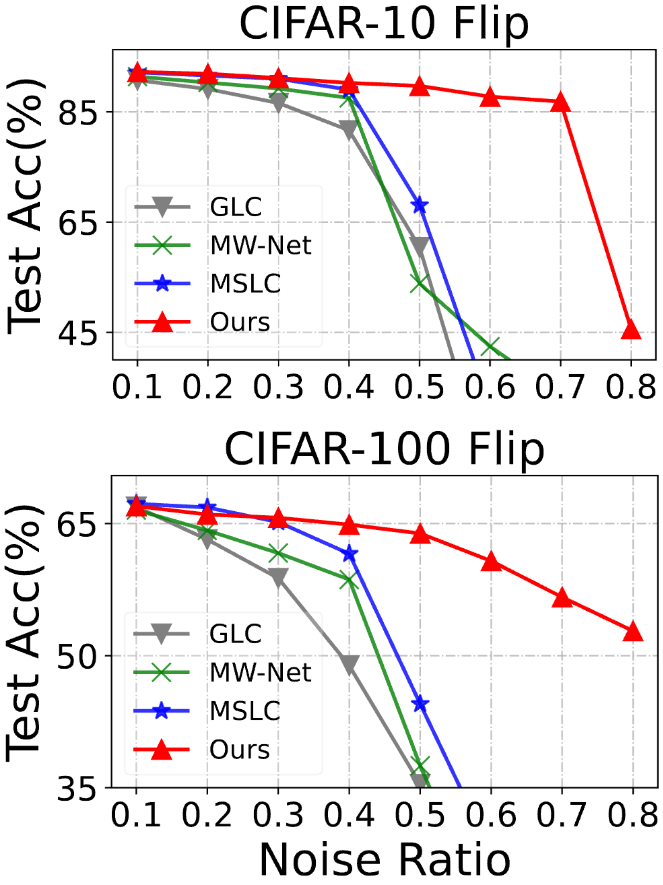}
\includegraphics[width=0.46\linewidth]{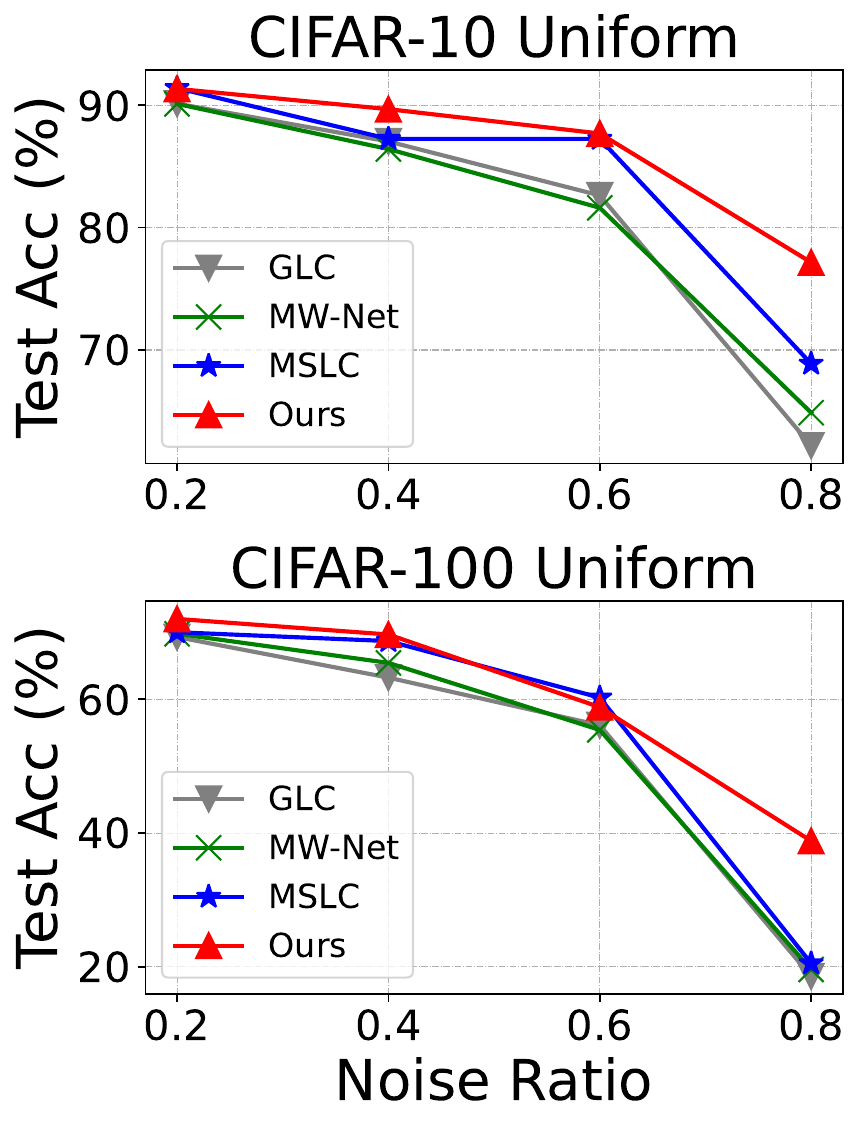}
\caption{The performance as the noise ratio increases is compared. VRI continues to deliver good performance even at significantly higher noise ratios.}
\label{fig:robustness}
\end{figure}

\begin{figure}[]
\centering
\includegraphics[width=.95\linewidth]{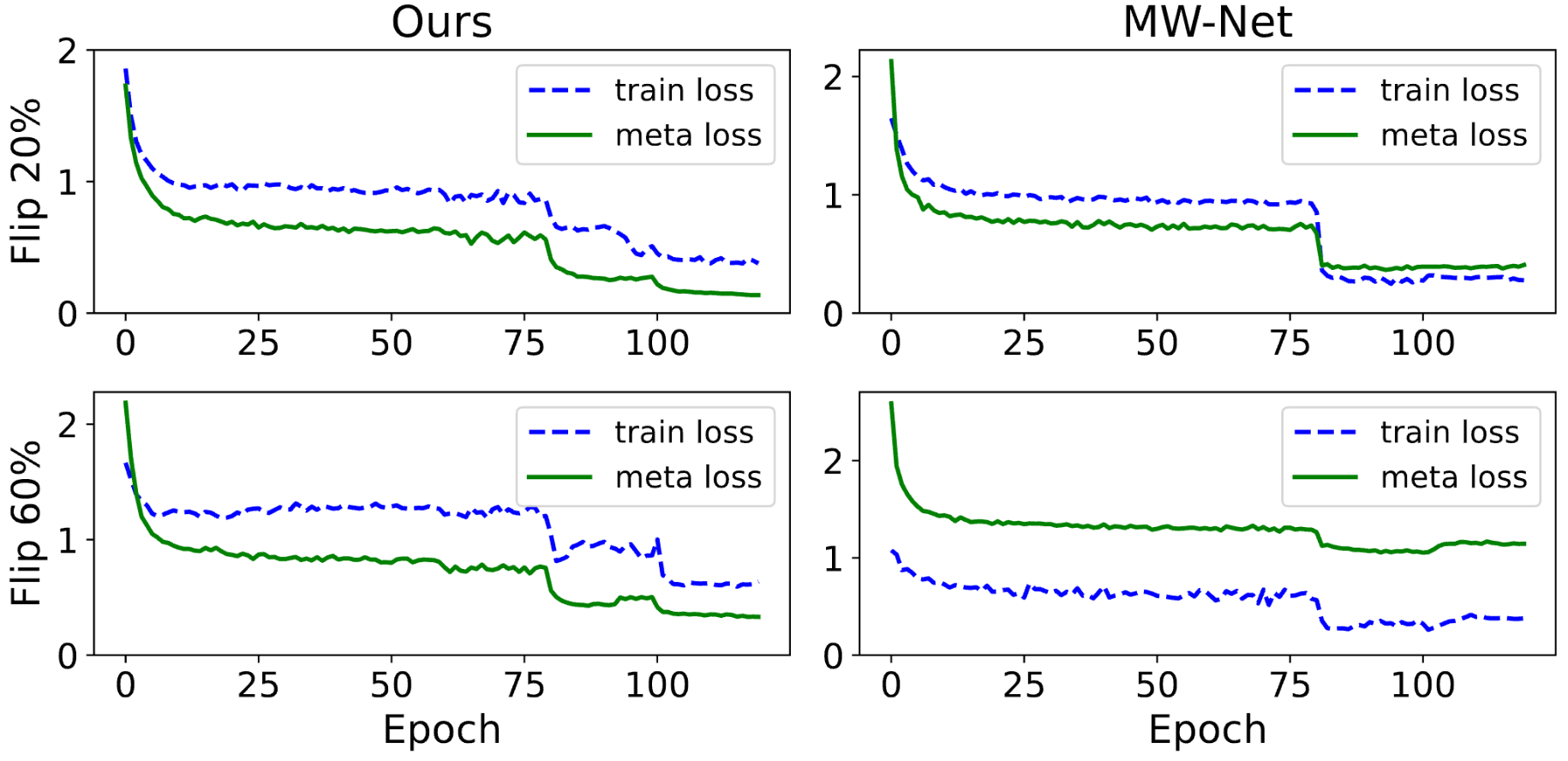} 
\caption{Our algorithm achieves a stable convergence and displays robustness on flip noise with a high ratio (e.g., $60\%$).}
\label{fig:more_ana2}
\end{figure}

\vspace{1mm}
\noindent\textbf{Real-world Noise}. To evaluate the performance on real-world noise, we conduct experiments on two large-scale real-world datasets, \textit{i.e.}, Clothing1M, and Food-101N, and choose the clean validation set as meta-data. For the fair comparison, we adopt the same evaluation protocol in~\citep{shu2019meta,zhangyikai2021learning} and use the same backbone of ResNet-50 pre-trained on ImageNet. We compare VRI with current SOTA methods. As shown in Table \ref{tab:cloth}, the proposed VRI achieves the highest accuracy of 75.19\% on Clothing1M and 86.24\% on Food-101N, consistently outperforming the homogeneous MC approximation method (\textit{e.g.}, WarPI). VRI also gains a large improvement of 1.4\% on Clothing1M and 1.5\% on Food-101N compared with other meta-learning methods (\textit{e.g.}, MW-Net), demonstrating its great effectiveness in real-world application.

Indeed, even for the state-of-the-art methods (\textit{e.g.}, DivideMix, ELR), they inevitably involve hyper-parameters and require a clean set (10$\%$ of training data, 5k samples of CIFAR) for cross-validation (CV). Our method is proposed to learn an adaptive rectifying strategy in a data-driving way, resolving the issue of \textit{scalability} in CV.

\vspace{1mm}
\noindent\textbf{Open-set Noise}. To evaluate the performance of VRI on open-set noise conditions, we follow the work \citep{yao2021jo} and manually construct this type of label noise, named CIFAR80N. Specifically, we first regard the last 20 categories in CIFAR100 as out-of-distribution ones. Then we create in-distribution noisy samples by randomly corrupting $\rho$ percentage of the remaining sample's labels. This finally leads to an overall noise ratio $\rho_{all} = 0.2 + 0.8 \rho$. 

We compare VRI with typical LNL methods and three methods of learning with openset noise \citep{yao2021jo,sun2022pnp,xu2023usdnl} on CIFAR-80N. As shown in Table \ref{tab:openset}, the performance of our VRI significantly surpasses that of the baseline methods. Notably, VRI achieves a 3.48\% improvement over the state-of-the-art method under the Flip 40\% noise condition.

\subsection{Further Analysis}

\noindent\textbf{Effectiveness}. To directly visualize the effect after rectification, we plot the distribution of training losses for all samples in Figure \ref{fig:vis} (a) when finishing the training process. The blue part represents the original loss without rectification, while the orange is for the loss computed from the rectified logits using our meta-network. As shown in Figure \ref{fig:vis} (a), the rectified loss is lower than the original one with high probability. 
The area of the original loss increases as the noise ratio rises, indicating the effect of rectification becomes more obvious.
To further illustrate its effectiveness, we adopt the prediction from the rectified logits as the clean label to estimate the transition matrix for constructing flip noise. We draw the initialized and estimated transition matrices for 20\%, 40\%, and 60\% ratios on CIFAR-10 in Figure \ref{fig:vis} (b). We almost achieve the unbiased estimation for the initialized matrix. 

To demonstrate the effectiveness in preventing model collapse, we conducted experiments on CIFAR-10 under 70\% uniform noise and plotted the norm of the Gaussian variance of the posterior in Fig. \ref{fig:collapse}. The norm of variance for the rectification vector of MC degenerates to zero in some cases, indicating a collapse to a deterministic model. This is also reflected in the degraded generalization performance of the test accuracy. In contrast, for VRI, the norm of the Gaussian variance remains stable around 3.

\begin{figure}[]
\centering
\includegraphics[width=.85\linewidth]{ 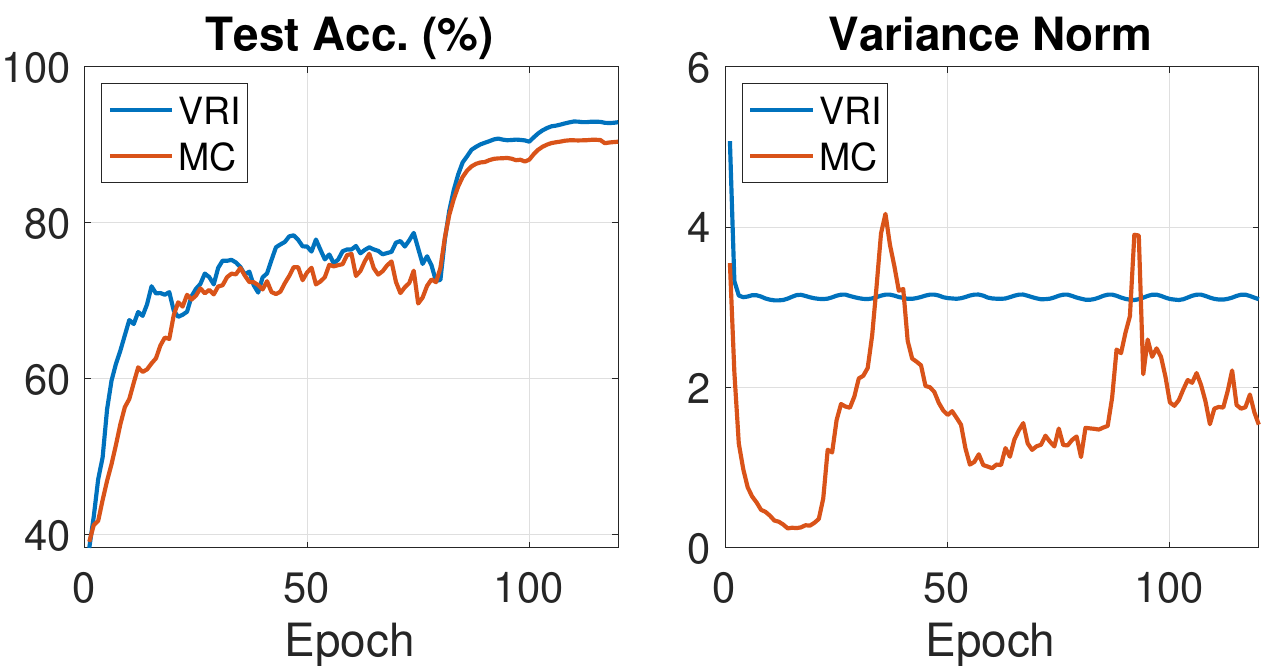} 
\caption{Model collapse of MC approximation.}
\label{fig:collapse}
\end{figure}

\noindent\textbf{Robustness}. 
We evaluate the generalization ability of VRI on more challenging conditions with high flip noise ratios. We compare VRI with three typical meta-learning methods, \textit{i.e.}, GLC~\citep{hendrycks2018using} of loss correction, MW-Net~\citep{shu2019meta} of reweighting, MSLC~\citep{yichen2020softlabel} of label correction. We adopt the same backbone network of ResNet-32 and a consistent setting of 1,000 meta samples. As shown in Figure \ref{fig:robustness}, VRI can still produce favorable results, even in challenging conditions with a far higher noise ratio. Compared with the SOTA meta-learning methods, VRI can retain the high accuracy of 86\% on CIFAR-10 with the setting of 70\% noise ratio.

We also plot Figure \ref{fig:more_ana2} about the training and meta-loss to explain this phenomenon. For other meta-learners (\textit{e.g.}, MW-Net), their meta-network might have limited ability to conduct the meta-learning process with a high ratio of flip noise. As the noise rate exceeds 50\%, the learning process in MW-Net is dominated by the classification network, where the empirical error decreases rapidly but the meta error still remains high. This renders non-convergence for optimizing the meta-network, leading to poor generalization performance. For MSLC, the backbone needs warm-up with the training data, which certainly degenerates the performance for high noise ratios. For VRI, our meta-network is powerful enough to rectify the training process by taking the feature and label as input and generating an effective rectifying vector, which is endowed with robustness to flip noise with high ratios.

\begin{figure}[t]
    \centering
    \includegraphics[width=0.98\linewidth]{ 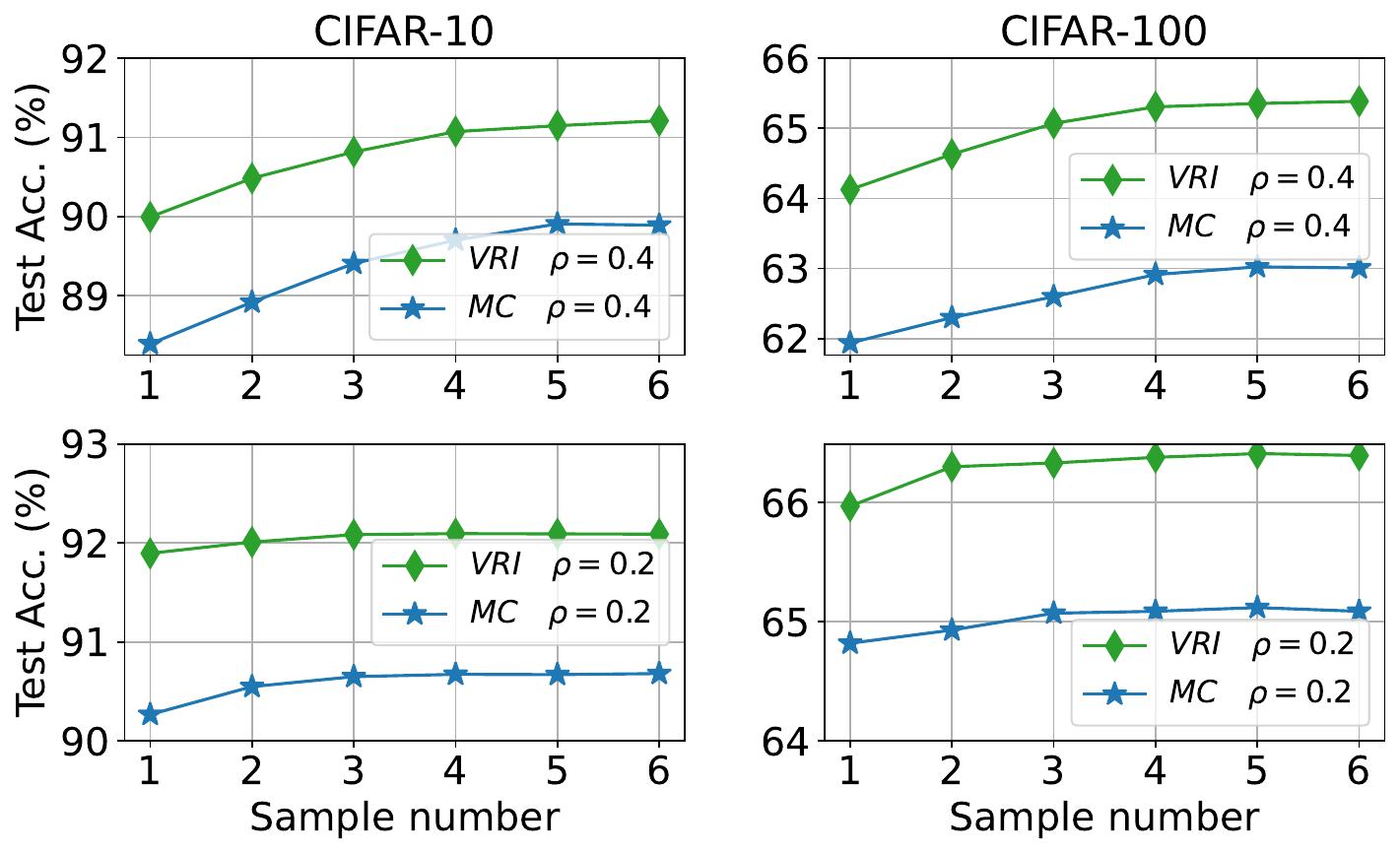}
    \caption{MC method is more sensitive about the sample number compared with our proposed VRI.}
    \label{fig:number_k}
\end{figure}

\subsection{Ablation Study} 

\noindent\textbf{Sampling number}. The sampling number $k$ in the Monte Carlo (MC) approximation has an impact on performance. We conduct experiments on CIFAR-10 and CIFAR-100 with variant $k$ for two flip noise ratios. As shown in Figure \ref{fig:number_k}, the testing accuracy for MC essentially turns out to be higher, then keeps stable as the sample number increases. Despite the gain of the performance from more samples, the training time increases linearly as illustrated in Table \ref{tab:efficiency}. Thanks to the variational term in VRI, we achieve higher accuracy than the MC approximation while keeping good efficiency.

\begin{table}[t]
    \centering
    \small
    \caption{VRI yields higher performance than MC approximation with an efficient inference.}
    \begin{tabular}{l|ccc}
    \toprule[0.9pt]
      & $\bm k$     & Time (min./epoch)  & Test Acc. (\%) \\ \midrule 
             &1         & 2.17               & 88.23  \\
MC           &3         & 4.32               & 89.45  \\
             &5         & 7.04               & 89.87  \\ \midrule
VRI          &1         & 2.20               & 90.20  \\
    \bottomrule[0.9pt]
\end{tabular}
\label{tab:efficiency}
\end{table}

\begin{table}[t]
    \centering
    \small
    \caption{Test accuracy (\%) of different architectures of the meta-net on CIFAR-10 Flip 40\%. $\mathcal{C}$ and $\hat{\mathcal{C}}$ denote the number of classes and the dimension of the embedding vector of the label, respectively.}
    \vspace{-2mm}
    \begin{tabular}{l|c}
    \toprule[1pt]
       Structure: $\{512+\hat{\mathcal{C}},h_1,...,h_n,\mathcal{C}\}$  & Accuracy \\ \midrule
    $\{512+\hat{\mathcal{C}},\, 128,\, \mathcal{C}\}$   &  89.76     \\
    $\{512+\hat{\mathcal{C}},\, 512,\, 512,\, \mathcal{C}\}$   &  90.69     \\
    $\{512+\hat{\mathcal{C}}, 1024, 512,\, \mathcal{C}\}$ {\small (ours)}   & 91.21 \\
    $\{512+\hat{\mathcal{C}}, 1024, 1024, \mathcal{C}\}$    & 90.89 \\
    $\{512+\hat{\mathcal{C}},\, 1024,\, 1024,\, 512,\,512,\, \mathcal{C}\}$                & 90.61   \\
    \bottomrule[1pt]
    \end{tabular}
\label{tab:metanet_structure}
\end{table}

\begin{table}[t]\label{tab:activation}
\caption{Test accuracy (\%) of different activation functions in the last layer of the meta-net on CIFAR-10 Unif 40\%. A ResNet-18 is used.}
\centering
\small
\setlength{\tabcolsep}{5mm}{
\begin{tabular}{l|c}
\toprule[1.1pt]
\multicolumn{1}{l|}{Homogeneous functions}     & Accuracy    \\ \midrule
Sigmoid  {\small (ours)}      & 91.29  \\
Tanh                           & 88.17\\
w/o Sigmoid                    & 86.90\\
\bottomrule[1.1pt]
\end{tabular}
}
\label{tab:output_layer}
\end{table}

\begin{table}[t]
\caption{Performance comparison (\%) of VRI and its correspondingly non-Bayesian version.} 
\centering
\small
\setlength{\tabcolsep}{3mm}{
\begin{tabular}{c|cc}
\toprule[1.1pt]
\multicolumn{1}{c|}{Noise}  & Method  & Test Acc.    \\ \midrule
CIFAR-10     & VRI                       & \textbf{91.29}  \\
Unif. 40\%   & Non-Bayesian VRI          & 89.27 \\ \midrule
CIFAR-100    & VRI                       & \textbf{68.92}  \\
Unif. 40\%   & Non-Bayesian VRI          & 66.96  \\ \midrule
CIFAR-10      & VRI                       & \textbf{90.60}  \\
Inst. 40\%                                        & Non-Bayesian VRI          & 87.01  \\ \midrule
CIFAR-100    & VRI                       & \textbf{68.17}  \\
Inst. 40\%                                        & Non-Bayesian VRI          & 64.92  \\ \bottomrule[1.1pt]
\end{tabular}
}
\label{tab:non_bayesian}
\end{table}

\begin{figure*}[t]
\centering
\subfigure[Sensitivity analysis of $\lambda$]{
\begin{minipage}[b]{0.28\linewidth}
\includegraphics[width=1\linewidth]{ 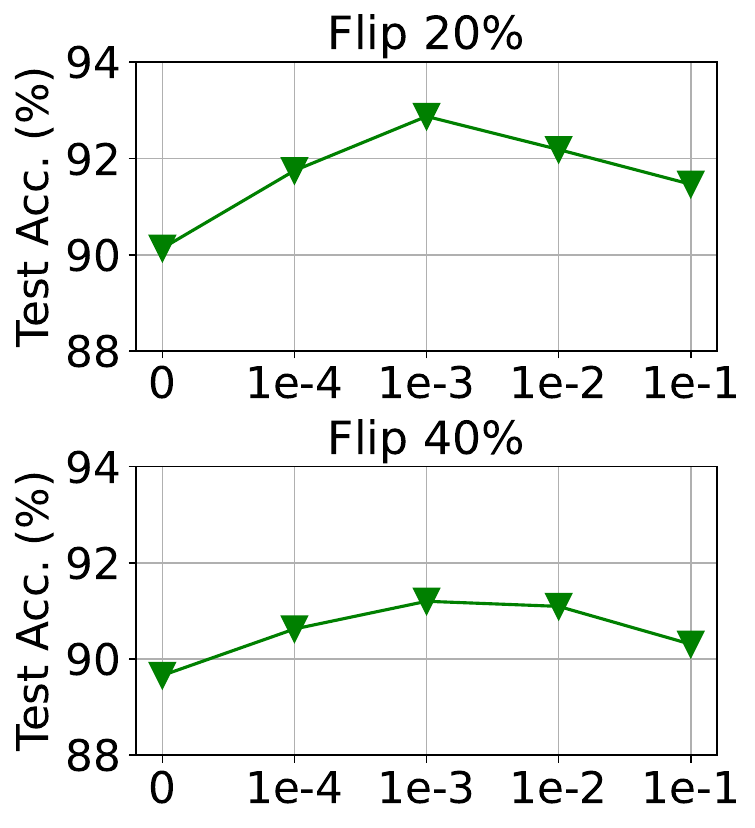}
\end{minipage}}
\hspace{5pt}
\subfigure[Sensitivity analysis of the number of the meta set]{
\begin{minipage}[b]{0.59\linewidth}
\includegraphics[width=1.0\linewidth]{ 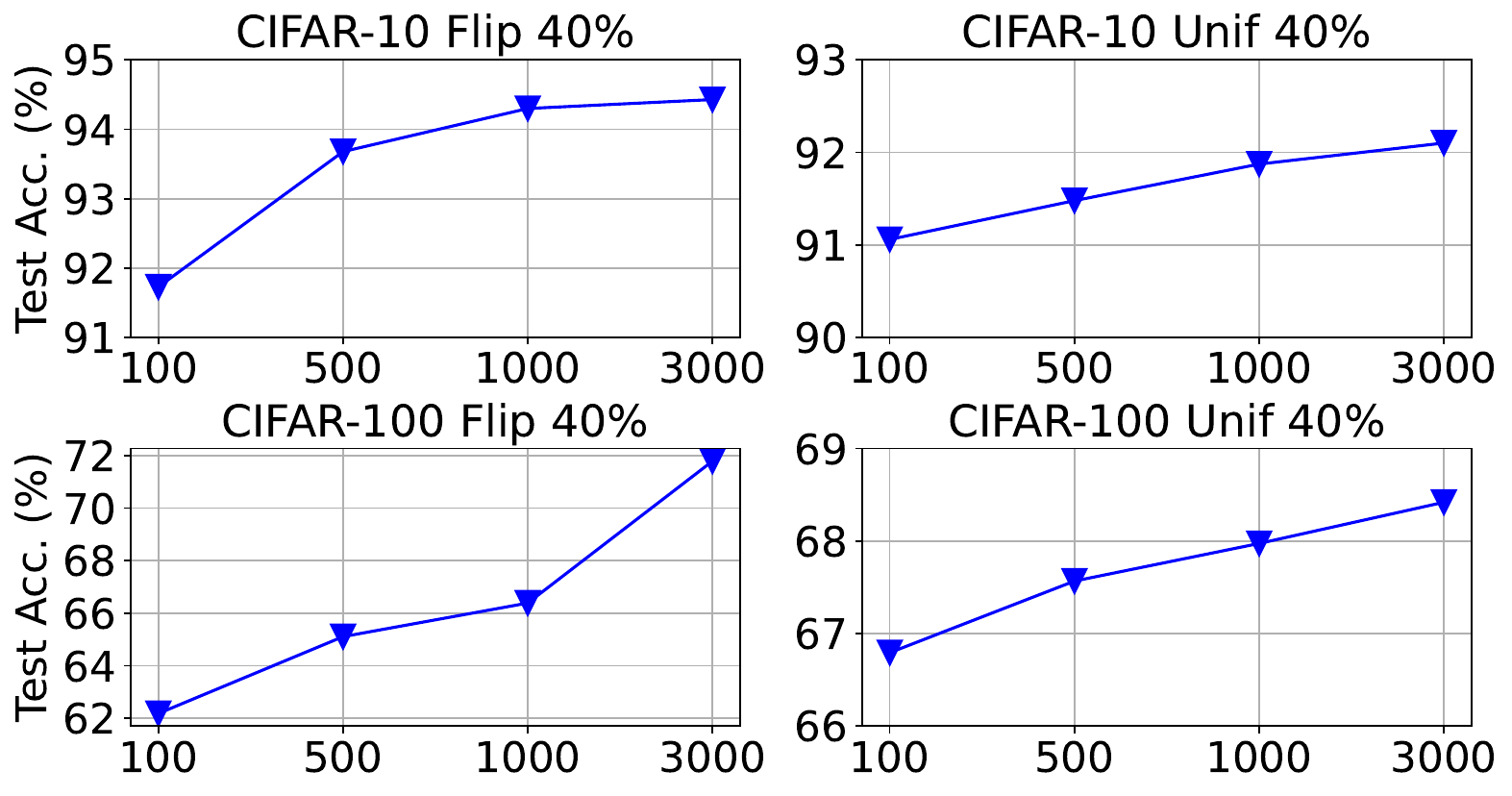}
\end{minipage}}
\caption{(a) We obtain the best performance on with $\lambda$ = 0.001. (b) The performance improves as the number of meta samples increases.}
\label{fig:sensitivity}
\end{figure*}

\vspace{1mm}
\noindent\textbf{Architectures and activation functions of the meta net}. The architecture of the meta-network has effects on performance as variations in depth and width lead to differing approximation abilities. As depicted in Table \ref{tab:metanet_structure}, 
the model achieves the best performance with two hidden layers containing 1024 and 512 neurons, potentially due to the limited capacity of smaller meta-networks and the risk of overfitting associated with larger architectures. The activation function in the final layer also significantly impacts performance. We replaced the Sigmoid function with tanh. As illustrated in Table \ref{tab:output_layer}, the model utilizing Sigmoid outperforms the one employing tanh. This discrepancy may stem from the larger scale of the output produced by tanh.

\begin{table*}[t]
	\centering
	\caption{Testing Accuracy (\%) on CIFAR-10 and CIFAR-100 with \textbf{uniform noise} (top) and \textbf{flip noise} (bottom) when the accessibility of meta-data is restricted.}
	\label{tab:unif_nometa}
\small
\scalebox{0.93}{
	\begin{tabular}{c|l|c|cc|cc|c}
		\toprule[1.1pt]
		&\multicolumn{2}{l|}{Dataset}    & \multicolumn{2}{c|}{CIFAR-10}                         & \multicolumn{2}{c|}{CIFAR-100}                   \\ \midrule
		&\multicolumn{1}{l|}{Noise Ratio}     &\multicolumn{1}{l|}{Structure}                 & \multicolumn{1}{c}{40\%} & \multicolumn{1}{c|}{60\%} & \multicolumn{1}{c}{40\%} & 60\%   & Average gap     \\ \midrule
\multirow{5}{*}{\rotatebox{90}{\textbf{Uniform}}}			&{Baseline}         &ResNet-18               & 68.07$\pm$1.23                & 53.12$\pm$3.03                & 51.11$\pm$0.42                & 30.92$\pm$0.33   &26.0 \\
		&\textbf{VRI} (-)         &ResNet-18               & 80.23$\pm$1.42                & 74.54$\pm$2.46                & 59.39$\pm$0.73                & 49.39$\pm$0.46  & 10.1 \\
		&\textbf{VRI} (+)         &ResNet-18               & 91.24$\pm$1.42                & 87.45$\pm$2.16                & 66.39$\pm$0.44                & 58.60$\pm$0.56 & 1.0  \\
		&\textbf{VRI} (Aug.)         &ResNet-18               & 90.78$\pm$1.56                & 87.98$\pm$2.12                & 66.78$\pm$0.68                & 58.79$\pm$0.76  &0.8 \\
		&\textbf{VRI} (Standard)  &ResNet-18       & \textbf{91.58$\pm$0.17}   & \textbf{88.68$\pm$0.22}       & \textbf{67.92$\pm$0.19}       & \textbf{59.32$\pm$0.31} & 0\\ \midrule[1.1pt]
\multirow{5}{*}{\rotatebox{90}{\textbf{Flip}}}	&  {Baseline}  &ResNet-32   & 76.83$\pm$0.32     & 70.77$\pm$2.31   & 50.86$\pm$0.27      & 43.01$\pm$1.16 & 18.2\\
	&\textbf{VRI} (-)         &ResNet-32                & 82.23$\pm$1.06                & 80.34$\pm$1.96                & 58.47$\pm$0.78                & 55.78$\pm$0.45  & 9.4 \\
	&	\textbf{VRI} (+)         &ResNet-32                & 90.88$\pm$1.16                & 90.36$\pm$1.84                & 65.47$\pm$0.81                & 64.36$\pm$0.55  & 0.8 \\
	&	\textbf{VRI} (Aug.)         &ResNet-32                & 91.11$\pm$1.12                & 90.34$\pm$1.87                & 65.67$\pm$0.98                & 64.24$\pm$0.56   & 0.5\\
	&	\textbf{VRI} (Standard)  &ResNet-32      & \textbf{91.93$\pm$0.14}    & \textbf{91.21$\pm$0.33}   & \textbf{66.03$\pm$0.21} & \textbf{65.04$\pm$0.38} & 0  \\
		\bottomrule[1.1pt]
	\end{tabular}
}
\end{table*}

\begin{table*}[]
  \centering
  \caption{Testing accuracy (\%) of VRI without given meta-data on ANIMAL-10N.}\label{tab:animal}
  \small
\scalebox{0.8}{
\begin{tabular}{cccccccc}
\toprule[1pt]
Baseline   & \citet{song2019selfie}  & \citet{zhangyikai2021learning}  & \citet{chen2021boosting} & \citet{englesson2021generalized} &  \citet{chen2022compressing}& \textbf{VRI} (-) & \textbf{VRI} (+) \\ \midrule
79.4        & 81.8 & 83.4 &84.1  & 84.2  & 84.5 & 81.4 & \textbf{85.8}\\
\bottomrule[1pt]
\end{tabular}
}
\end{table*}

\vspace{1mm}
\noindent\textbf{The cardinality of the meta set}. The cardinality of the meta set has an impact on the performance. We set it to 1,000 for CIFAR datasets as other meta-learning methods (\textit{e.g.}, MWNet, MSLC). We also study the influence in Figure \ref{fig:sensitivity} (b). The performance improves as the number of meta samples increases, especially for flip noise. Also, VRI can obtain considerable performance (91.07$\%$, CIFAR-10, flip 40$\%$) given limited meta samples (100). Here, the backbone is ResNet-18.

\vspace{1mm}
\noindent\textbf{Hyper-parameter discussion}. To illustrate the sensitivity of $\lambda$, we conduct experiments on CIFAR-10 under flip noise. As shown in Figure \ref{fig:sensitivity} (a), we obtain the best performance with $\lambda=0.001$. The accuracy would slightly drop as $\lambda$ increases. Indeed, we observe that the KL divergence of the variational term usually produces a large value at the beginning of the training, which would lead to an unstable learning process. Therefore, we set $\lambda$ as $0.001$ via cross-validation. The result also demonstrates that we can gain considerable improvement in performance by introducing the variational term.

\noindent
\textbf{Compared with a non-Bayesian form}. We eliminated the prior network and replaced the meta-network with one that directly generates the rectification vector $\mathbf{v}$. This vector is subsequently multiplied by the features computed by the classification network, leading to a non-Bayesian rectification process. We compared the performance of this framework with our proposed VRI and have documented the results in Table \ref{tab:non_bayesian}. In settings of Unif. 40\% and Inst. 40\%, our proposed VRI, a Bayesian noise-robust framework, consistently outperforms non-Bayesian methods in test accuracy on CIFAR-10 \& 100. These findings highlight the advantages of a Bayesian rectification approach in countering the negative impact of noisy labels.

\subsection{Learning without Meta-Data}
To evaluate the performance of the model when there is a lack of clean meta-data, we adopt the sample selection strategy~\citep{han2018co} to select reliable samples in the corrupted training set and treat them as pseudo meta-data. Specifically, we firstly conduct warming-up (CIFAR-10: 10 epochs. CIFAR-100: 30 epochs. ANIMAL-10N: 100 epochs) for the classification network to achieve the basic discrimination ability. Then, we apply the small-loss strategy and select 1,000 samples with higher confidence for each epoch. Next, we train our meta-network with the selected samples by using the proposed learning Algorithm \ref{alg:1}. The whole process can be summarized as Algorithm \ref{alg:2} in the Appendix \ref{app:b}.

For synthetic noise, the class distribution has an impact on the performance. We conduct two experiments. a) ``\textbf{VRI ($+$)}", balancing the class of selected metadata; b) ``\textbf{VRI ($-$)}", directly using the selected samples with the top 1,000 smallest losses. We observe that classes of the latter are extremely imbalanced. Besides, data augmentation \footnote{We utilize a robust image augmentation policy, RandAugmentMC \citep{cubuk2020randaugment}. During each training iteration, two strategies are randomly selected for image transformation. Importantly, strong augmentation is applied exclusively to the (noisy) training data, and not to the meta data.}  can also relieve the class-imbalanced issue. We select all training samples with the smaller loss via Gaussian Mixture Model clustering and use mixup to enhance training $\slash$ meta-data. 

As shown in Table \ref{tab:unif_nometa}, the performance heavily degenerates with imbalanced pseudo meta-data. Besides, the meta-learning framework without meta-data still outperforms the baseline that is directly trained on the noisy dataset and achieves favorable performance.

For real-world noise, VRI achieves the highest accuracy without meta data on the ANIMAL-10N dataset (Table \ref{tab:animal}). We adopt the same architecture of VGG19 as \citep{song2019selfie,zhangyikai2021learning,chen2021boosting}. To build the meta set, We first train the VGG19 for 100 epochs in a standard manner. We then use this network to select clean samples with the top 1,000 smallest empirical losses as meta data and carefully balance the number (100) for each class. Once we split the original training set into the noisy training set and meta set, we meta-learn a new VGG19 network from \textit{scratch} via VRI for evaluation.

\section{Conclusion}\label{sec:conclusions}
In this work, we propose variational rectification inference (VRI) for learning with label noise to tackle model collapse in the MC meta-learning method. VRI is built as a hierarchical Bayes to estimate the conditional predictive distribution and formulated as the variational inference problem. To achieve adaptively rectifying the loss with noisy labels, we design a meta-network, which is endowed with the ability to exploit information lying in the feature space. Our method can also meta-learn the rectifying process via bi-level programming, whose convergence can be theoretically guaranteed. To evaluate the effectiveness of VRI, we conduct extensive experiments on varied noise types and achieve competitive performance on those benchmarks. Experimental results demonstrate that VRI outperforms the MC method with low sampling rates, resulting in a more efficient learning process. To further boost our framework, we integrate the adaptive sample strategy into VRI and obtain comparable performance without meta data, beyond the common setting of existing meta-learning methods.

\begin{appendices}
\section{}\label{secA1}

\subsection{Derivations of The ELBO \label{de_ELBO}} 
For a singe observation $(\mathbf{x}, \mathbf{y} )$, the ELBO can be derived from the perspective of the KL divergence between the variational posterior $q_{\phi}(\mathbf{v}| \mathbf{x}, \mathbf{y})$ and the posterior $p(\mathbf{v}| \mathbf{x}, \mathbf{y})$:
 \begin{equation}\label{eq:bayes} 
\small
\begin{aligned}
	& \KL[q_{\phi}(\mathbf{v}| \mathbf{x}, \mathbf{y}) || p(\mathbf{v}| \mathbf{x}, \mathbf{y})]  \\
	& = \mathbb{E}_{q_{\phi}(\mathbf{v}| \mathbf{x}, \mathbf{y})} \left[\log q_{\phi}(\mathbf{v}| \mathbf{x}, \mathbf{y}) - \log p(\mathbf{v}| \mathbf{x}, \mathbf{y})\right]  \\
	& = \mathbb{E}_{q_{\phi}(\mathbf{v}| \mathbf{x}, \mathbf{y})} \left[\log q_{\phi}(\mathbf{v}| \mathbf{x}, \mathbf{y}) - \log \frac{p(\mathbf{v}| \mathbf{x}, \mathbf{y}) p(\mathbf{x}, \mathbf{y})}{p(\mathbf{x}, \mathbf{y})}\right]    \\
	& = \log p(\mathbf{y}| \mathbf{x}) + \mathbb{E}_{q_{\phi}(\mathbf{v}| \mathbf{x}, \mathbf{y})} \big[\log q_{\phi}(\mathbf{v}| \mathbf{x}, \mathbf{y})  \\ 
	  & \quad\quad\quad\quad\quad\quad\quad\quad  - \log p(\mathbf{y}| \mathbf{x}, \mathbf{v}) - \log p(\mathbf{v}| \mathbf{x})  \big] \\
	& = \log p(\mathbf{y}| \mathbf{x}) - \mathbb{E}_{q_{\phi}(\mathbf{v}| \mathbf{x}, \mathbf{y})} \left[\log p(\mathbf{y}| \mathbf{x}, \mathbf{v})\right]  \\ 
	   & \quad\quad\quad\quad\quad\quad\quad\quad  +\KL[q_{\phi}(\mathbf{v}| \mathbf{x}, \mathbf{y}) || p(\mathbf{v}| \mathbf{x})]   \\
	& \geq 0. 
 \end{aligned}
\end{equation}

Specifically, we apply Bayes' rule to derive Eq. (\ref{eq:bayes}) as
 \begin{equation}
\small
\begin{aligned}
	p &(\mathbf{v}| \mathbf{x}, \mathbf{y}) = \frac{p(\mathbf{v}| \mathbf{x}, \mathbf{y}) p(\mathbf{x}, \mathbf{y})}{p(\mathbf{x}, \mathbf{y})}   \\ 
	&= \frac{p(\mathbf{y} | \mathbf{x},\mathbf{v} ) p(\mathbf{x},\mathbf{v})}{p(\mathbf{x}, \mathbf{y})}  = \frac{p(\mathbf{y} | \mathbf{x},\mathbf{v} ) p(\mathbf{v}|\mathbf{x})}{p(\mathbf{y}| \mathbf{x})}.
 \end{aligned}
\end{equation}

%The prior $p(\mathbf{v}| \mathbf{x})$ can be assigned to a certain distribution of $p(\mathbf{v}) = \mathcal{N}(0 ,\text{I})$ as~\citep{zhang2019variational}.
Therefore, the ELBO for the log-likelihood of the predictive distribution in Eq. (\ref{eq:elbo}) can be written as follows
 \begin{equation}
\small
\begin{aligned}
	& \log p(\mathbf{y}| \mathbf{x})  \\
	&  \geq  \mathbb{E}_{q_{\phi}(\mathbf{v}| \mathbf{x}, \mathbf{y})} \left[\log p(\mathbf{y}| \mathbf{x}, \mathbf{v})\right]  - \KL[q_{\phi}(\mathbf{v}| \mathbf{x}, \mathbf{y}) ||  p(\mathbf{v}| \mathbf{x}))]  \\ & \quad  = \mathcal{L}_{\text{ELBO}}.
 \end{aligned}
\end{equation}

\subsection{Proof\label{app:a}}

\noindent \textbf{Lemma \ref{lemma1} (Smoothness)}

\begin{proof}

We begin with computation of the derivation of the meta loss $\widetilde{\mathcal{L}}^{emp}(\hat{\theta}) $ \textit{w.r.t.} the meta-network $\phi$. By using Eq. (\ref{eq:theta_step}), we have

 \begin{equation}\label{1storder}
\small
\begin{aligned}
		 & \frac{\partial  \mathcal{L}^{meta}(\hat{\theta})}{\partial \phi}   = \frac{\partial \mathcal{L}^{meta}(\hat{\theta})}{\partial \hat{\theta}}  \frac{\partial \hat{\theta}}{\partial V(\phi)}   \frac{\partial V(\phi) }{\partial \phi} \\
		& \qquad \quad  =\alpha \frac{\partial \mathcal{L}^{meta}(\hat{\theta})}{\partial \hat{\theta}}  \left( \nabla_\theta L(\theta) + \frac{\partial  \KL }{\partial V(\phi)} \right)  \frac{\partial V(\phi) }{\partial \phi}.
 \end{aligned}
\end{equation}
To simplify the proof, we neglect Monte Carlo estimation in Eq. \ref{eq:obj1} and consider it as a deterministic rectified vector in the following. This would not affect the result since there ultimately exists a rectified vector for computing the expectation of those sampled losses.
Taking the gradient of $\phi$ on both side of Eq. (\ref{1storder}),
\begin{equation}
\small
\begin{aligned}
	 &\frac{\partial^2  \mathcal{L}^{meta}(\hat{\theta})}{\partial \phi^2}  \\ &\quad =  \underbrace{\alpha  \frac{\partial }{\partial \phi} \left( \frac{\partial \mathcal{L}^{meta}(\hat{\theta})}{\partial \hat{\theta}}  \left( \nabla_\theta L(\theta) + \frac{\partial  \KL }{\partial V(\phi)} \right) \right) \frac{\partial V(\phi) }{\partial \phi}}_{\text{\ding{182}}} \\ & + \underbrace{ \alpha \frac{\partial \mathcal{L}^{meta}(\hat{\theta})}{\partial \hat{\theta}}  \left( \nabla_\theta L(\theta) + \frac{\partial  \KL }{\partial V(\phi)} \right)  \frac{\partial^2 V(\phi) }{\partial \phi ^2} }_{\text{\ding{183}}}.    \end{aligned}
\end{equation}
For the first term \ding{182} in the right hand, we can obtain the following inequality \textit{w.r.t.} its norm
\begin{equation}
\small
\begin{aligned}
		\|\text{\ding{182}}\|  & \le  \alpha \delta  \left\|     \frac{\partial }{\partial \hat{\theta}} \left( \frac{\partial \mathcal{L}^{meta}(\hat{\theta})}{    \partial \phi  } \right) \left( \nabla_\theta L(\theta) + \frac{\partial  \KL }{\partial V(\phi)}  \right)   \right \|  \\ 
		& = \alpha^2 \delta  \Bigg \|     \frac{\partial }{\partial \hat{\theta}} \left(  \frac{\partial \mathcal{L}^{meta}(\hat{\theta})}{\partial \hat{\theta}}  \left( \nabla_\theta L(\theta) + \frac{\partial  \KL }{\partial V(\phi)} \right)  \frac{\partial V(\phi) }{\partial \phi} \right) \\ 
		& \qquad\qquad \quad \left( \nabla_\theta L(\theta) + \frac{\partial  \KL }{\partial V(\phi)}  \right)   \Bigg  \|  \\
		& = \alpha^2 \delta  \Bigg \|     \frac{\partial^2 \mathcal{L}^{meta}(\hat{\theta})}{\partial \hat{\theta}^2}  \left( \nabla_\theta L(\theta) + \frac{\partial  \KL }{\partial V(\phi)} \right)^2  \frac{\partial V(\phi) }{\partial \phi} \Bigg  \| \\
		& \le \ell \alpha^2 \delta^2  (\tau + o)^2,  
     \end{aligned}
\end{equation}
since we assume $ \|  \frac{\partial^2 \mathcal{L}^{meta}(\hat{\theta})}{\partial \hat{\theta}^2} \|  \le \ell$,  $ \|  \nabla_\theta L(\theta) \|  \le \tau$,  $ \|  \frac{\partial  \KL }{\partial V(\phi)}  \|  \le o$, and $ \| \frac{\partial V(\phi) }{\partial \phi}   \|  \le \delta$. 

For the second term \ding{183}, we can also obtain
\begin{equation}
\small
\begin{aligned}
		\|\text{\ding{183}}\|  & \le  \alpha \tau (\tau + o) \zeta  
   \end{aligned}
\end{equation}
with the assumption $ \|  \frac{\partial^2 V(\phi) }{\partial \phi ^2} \| \le \zeta$.
Therefore, we have 
\begin{equation}
\small
\begin{aligned}
		\left \| \frac{\partial^2  \mathcal{L}^{meta}(\hat{\theta})}{\partial \phi^2}  \right\|   \le \alpha(\tau+o)\left( \ell \alpha \delta^2(\tau+o) + \tau \zeta   \right).  \end{aligned}
\end{equation}
Let $\hat{\ell} = \alpha(\tau+o)\left( \ell \alpha \delta^2(\tau+o) + \tau \zeta \right)$, we can conclude the proof that 
\begin{equation} 
\small
\begin{aligned}
  \|  \mathcal{L}^{meta}(\hat{\theta}(\phi^{(t+1)})) -  \mathcal{L}^{meta}(\hat{\theta}(\phi^{(t)}))     \|  
\le  \hat{\ell}   \|   \phi^{(t+1)} - \phi^{(t)}      \|.  
 \end{aligned}
\end{equation}

\end{proof}

\noindent \textbf{Theorem \ref{th1} (Convergence Rate)}

\begin{proof}
Consider 
%\begin{align}\label{decom}
%	\begin{split}
\begin{equation}\label{decom}
\small
\begin{aligned}
		&\mathcal{L}^{meta}(\hat{\theta}^{(t+1)}(\phi^{(t+1)}))-\mathcal{L}^{meta}(\hat{\theta}^{(t)}(\phi^{(t)})) \\
		&= \underbrace{\mathcal{L}^{meta}(\hat{\theta}^{(t+1)}(\phi^{(t+1)}))- \mathcal{L}^{meta}(\hat{\theta}^{(t)}(\phi^{(t+1)}))}_{\text{\ding{184}}}
		\\
		&\qquad+\underbrace{\mathcal{L}^{meta}(\hat{\theta}^{(t)}(\phi^{(t+1)}))-\mathcal{L}^{meta}(\hat{\theta}^{(t)}(\phi^{(t)}))}_{\text{\ding{185}}}. 
   \end{aligned}
\end{equation}
  %\nonumber
	%\end{split}
%\end{align}

For \ding{184}, by Lipschitz smoothness of the meta loss function for $\theta$, we have 
\begin{equation}
\small
\begin{aligned}
	&\mathcal{L}^{meta}(\hat{\theta}^{(t+1)}(\phi^{(t+1)}))- \mathcal{L}^{meta}(\hat{\theta}^{(t)}(\phi^{(t+1)}))   \\
	&\leq \langle \nabla \mathcal{L}^{meta}(\hat{\theta}^{(t)}(\phi^{(t+1)})), \hat{\theta}^{(t+1)}(\phi^{(t+1)})-\hat{\theta}^{(t)}(\phi^{(t+1)}) \rangle   \\
	&\qquad +\frac{\ell}{2}\|\hat{\theta}^{(t+1)}(\phi^{(t+1)})-\hat{\theta}^{(t)}(\phi^{(t+1)})\|_2^2 .
 \end{aligned}
\end{equation}

We firstly write $\hat{\theta}^{(t+1)}(\phi^{(t+1)}), \hat{\theta}^{(t)}(\phi^{(t+1)})$  with Eq. (\ref{eq:theta_step}). Using Eq. (\ref{eq:theta_update}), we obtain 
\begin{equation}
\small
\begin{aligned}
	\hat{\theta}^{(t+1)}&(\phi^{(t+1)}) -  \hat{\theta}^{(t)}(\phi^{(t+1)}) \\ &=  - \alpha  \nabla_\theta \mathcal{L}^{emp}(\hat{\theta}^{(t+1)}(\phi^{(t+1)})). 
 \end{aligned}
\end{equation}
and
\begin{equation}
\small
\begin{aligned}
	\|  \mathcal{L}^{meta}&(\hat{\theta}^{(t+1)}(\phi^{(t+1)}))- \mathcal{L}^{meta}(\hat{\theta}^{(t)}(\phi^{(t+1)})) \|   \\ 
	&\leq \alpha_t \tau^2+ \frac{\ell \alpha_t^2}{2} \tau^2 = \alpha_t\tau^2 (1+\frac{\alpha_t \ell}{2}),
 \end{aligned}
\end{equation}
since $\left\|\frac{\partial L(\theta)}{\partial \theta}\Big|_{\theta^{(t)}}\right\|\leq \tau$, $\left\|\!\frac{\partial L_i^{meta} (\hat{\theta})}{\partial \hat{\theta}}\!\Big|_{\hat{\theta}^{(t)}}\!\right\|\!\!\leq\!\! \tau$, and the output of $V(\cdot)$ is bounded with the sigmoid function.

For \ding{185}, since the gradient is computed from a mini-batch of training data that is drawn uniformly, we denote the bias of the stochastic gradient $\varepsilon^{(t)} = \nabla \widetilde{\mathcal{L}}^{meta}\left(\hat{\theta}^{(t)}\left(\phi^{(t)}\right)\right)- \nabla\mathcal{L}^{meta}\left(\hat{\theta}^{(t)}\left(\phi^{(t)}\right)\right) $. We then observe its expectation obeys $\mathbb{E}[\varepsilon^{(t)}]  = 0$ and its variance obeys $\mathbb{E}[\| \varepsilon^{(t)}\|_2^2 ] \leq \sigma^2 $. 

By smoothness of $\nabla\mathcal{L}^{meta}(\hat{\theta}^{(t)}(\phi))$ for $\phi$ in Lemma \ref{lemma1}, we have
\begin{equation}
\small
\begin{aligned}
	&\mathcal{L}^{meta}(\hat{\theta}^{(t)}(\phi^{(t+1)}))-\mathcal{L}^{meta}(\hat{\theta}^{(t)}(\phi^{(t)}))  \\ 
	& \leq \langle\nabla \mathcal{L}^{meta}(\hat{\theta}^{(t)}(\phi^{(t)})),\phi^{(t+1)}-\phi^{(t)} \rangle + \frac{\hat{\ell}}{2} \|\phi^{(t+1)}-\phi^{(t)}\|_2^2   \\ 
	& = \langle\nabla \mathcal{L}^{meta}(\hat{\theta}^{(t)}(\phi^{(t)})), -\eta_t [\nabla \mathcal{L}^{meta}(\hat{\theta}^{(t)}(\phi^{(t)}))+\varepsilon^{(t)} ] \rangle \\
	& \quad + \frac{\hat{\ell} \eta_t^2}{2} \|\nabla \mathcal{L}^{meta}(\hat{\theta}^{(t)}(\phi^{(t)}))+\varepsilon^{(t)}\|_2^2   \\
	& = -(\eta_t-\frac{\hat{\ell} \eta_t^2}{2}) \|\nabla \mathcal{L}^{meta}(\hat{\theta}^{(t)}(\phi^{(t)}))\|_2^2 + \frac{\widetilde{\ell} \eta_t^2}{2}\|\varepsilon^{(t)}\|_2^2  \\
	& \quad- (\eta_t-\hat{\ell} \eta_t^2)\langle \nabla \mathcal{L}^{meta}(\hat{\theta}^{(t)}(\phi^{(t)})),\varepsilon^{(t)}\rangle.
 \end{aligned}
\end{equation}
Thus, Eq.(\ref{decom}) satisfies
\begin{equation}
\small
\begin{aligned}
    \label{decomieq}
		&\mathcal{L}^{meta}(\hat{\theta}^{(t+1)}(\phi^{(t+1)}))-\mathcal{L}^{meta}(\hat{\theta}^{(t)}(\phi^{(t)})) \\ 
		&\leq \alpha_t\tau^2 (1+\frac{\alpha_t \ell}{2}) -(\eta_t-\frac{\hat{\ell}\eta_t^2}{2}) \|\nabla \mathcal{L}^{meta}(\hat{\theta}^{(t)}(\phi^{(t)}))\|_2^2 \\
		&\quad + \frac{\hat{\ell}\eta_t^2}{2}\|\varepsilon^{(t)}\|_2^2 - (\eta_t-\hat{\ell} \eta_t^2)\langle \nabla \mathcal{L}^{meta}(\hat{\theta}^{(t)}(\phi^{(t)})),\varepsilon^{(t)}\rangle.
 \end{aligned}
\end{equation}

We take the expectation \textit{w.r.t.} $\varepsilon^{(t)}$ over Eq. (\ref{decomieq}) and sum up $T$ inequalities. By the property of the bias $\varepsilon^{(t)}$, we can obtain
\begin{equation}
\small
\begin{aligned}
    \sum\limits_{t = 1}^T & \left(\mathop{ \mathbb{E}}_{\varepsilon^{(t)}} \mathcal{L}^{meta}(\hat{\theta}^{(t+1)}(\phi^{(t+1)}))-\mathop{ \mathbb{E}}_{\varepsilon^{(t)}}\mathcal{L}^{meta}(\hat{\theta}^{(t)}(\phi^{(t)}))\right) \\
	&\leq  \tau^2  \sum\limits_{t = 1}^T \alpha_t(1+\frac{\alpha_t \ell}{2})   \\
	& \quad -  \sum\limits_{t = 1}^T (\eta_t-\frac{\hat{\ell}\eta_t^2}{2})\mathop{ \mathbb{E}}_{\varepsilon^{(t)}} \left[ \|\nabla \mathcal{L}^{meta}(\hat{\theta}^{(t)}(\phi^{(t)}))\|_2^2\right]    \\
	& \quad + \frac{\hat{\ell} \sigma^2}{2}  \sum\limits_{t = 1}^T \eta_t^2.
 \end{aligned}
\end{equation}
Taking the total expectation and reordering the terms, we have
\begin{equation}
\small
\begin{aligned}
	&\frac{1}{T}\sum\limits_{t = 1}^T (\eta_t-\frac{\hat{\ell}\eta_t^2}{2}) \mathop{ \mathbb{E}} \left[ \|\nabla \mathcal{L}^{meta}(\hat{\theta}^{(t)}(\phi^{(t)}))\|_2^2 \right]   \\
	&\leq   \frac{\mathcal{L}^{meta}(\hat{\theta}^{(0)}(\phi^{(0)}))  -  \mathop{ \mathbb{E}}\left[ \mathcal{L}^{meta}(\hat{\theta}^{(T+1)}(\phi^{(T+1)}) ) \right] }{T}   \\
	&+  \frac{\tau^2 }{T} \sum\limits_{t = 1}^T \alpha_t(1+\frac{\alpha_t \ell}{2}) + \frac{\hat{\ell} \sigma^2}{2T}  \sum\limits_{t = 1}^T \eta_t^2.
 \end{aligned}
\end{equation}
Let 
\begin{equation}
\small
\begin{aligned}
E=\mathcal{L}^{meta}(\hat{\theta}^{(0)}(\phi^{(0)}))  -  \mathop{ \mathbb{E}}\left[ \mathcal{L}^{meta}(\hat{\theta}^{(T+1)}(\phi^{(T+1)}) ) \right].
 \end{aligned}
\end{equation}
With the assumption of $\eta_t =\min\{\frac{1}{\hat{\ell}},\frac{C}{\sigma\sqrt{T}}\}$ and $\alpha_t=\min\{1,\frac{\kappa}{T}\}$, we have $\eta_t-\frac{\hat{\ell}\eta_t^2}{2} \ge \eta_t - \frac{\eta_t}{2} =  \frac{\eta_t}{2}$ and 
\begin{equation}
\small
    \begin{aligned}
    &\frac{1}{T}\sum\limits_{t = 1}^T  \mathop{ \mathbb{E}} \left[ \|\nabla \mathcal{L}^{meta}(\hat{\theta}^{(t)}(\phi^{(t)}))\|_2^2 \right]  \\
	&\leq   \frac{2E}{T\eta_1} +  \frac{(2+\ell)\tau^2  \alpha_1 }{\eta_1}  + \hat{\ell} \sigma^2 \eta_1  \\
	& =  \frac{2E}{T} \max\{\hat{\ell}, \frac{\sigma\sqrt{T}}{C} \}   +  (2+\ell) \tau^2  \min\{1,\frac{\kappa}{T}\} \max\{\hat{\ell}, \frac{\sigma\sqrt{T}}{C} \}  \\
	& \quad+ \hat{\ell} \sigma^2 \min\{\frac{1}{\hat{\ell}},\frac{C}{\sigma\sqrt{T}}\}  \\
	& \leq  \frac{2\sigma E}{C\sqrt{T}}    +  \frac{ (2+\ell)\tau^2 \kappa \sigma}{C \sqrt{T} } + \frac{C\hat{\ell} \sigma^2}{\sigma\sqrt{T}} = \mathcal{O}(\frac{1}{\sqrt{T}}). r
    \end{aligned}
\end{equation}
Thus, we conclude our proof.

\end{proof}

\subsection{Algorithm for VRI without the meta set\label{app:b}}
\begin{algorithm}[h]
\small
\caption{Learning without meta data}
\label{alg:2}
\begin{algorithmic}[1]
\REQUIRE 
Training set $\mathcal{D}_N$, number of meta samples $M$, batch size $n, m$, outer iterations $T$ for each epoch, sampling number $k$, step size $\alpha$, $\eta$, warming-up epoch $K$, training epoch $C$
\ENSURE Optimal $\theta^*$ \\
\STATE  Initialize parameters $\theta^{(0)}$, $\phi^{(0)}$, and $\omega^{(0)}$ \\
\STATE  Warm up parameters $\theta$ for $K$ epochs \\
\FOR{$c \in \{1,\dots,C\}$}
  \STATE $ \mathcal{D}_M = \text{SelectWithBalance}(\mathcal{D}_N, M)$
  \FOR{$t \in \{1,\dots,T\}$}
        \STATE $ \text{SampleBatch}(\mathcal{D}_N, n), \text{SampleBatch}(\mathcal{D}_M, m)$\\
        \STATE Form learning process of $\hat{\theta}^{(t)}(\phi, \omega)$   \\
        \STATE Optimize $\phi^{(t)}$ with $\hat{\theta}^{(t)}(\phi)$  \\ 
        \STATE Optimize $\omega^{(t)}$ with $\hat{\theta}^{(t)}(\omega)$  \\ 
        \STATE Optimize $\theta^{(t)}$ using the updated $\phi^{(t+1)}$  \\ 
  \ENDFOR \\
\ENDFOR \\
\end{algorithmic}
\end{algorithm}

\end{appendices}

\section*{Declarations}

\begin{itemize}
\item \textbf{Funding} This research was supported by Young Expert of Taishan Scholars in Shandong Province (No. tsqn202312026), Natural Science Foundation of China (No. 62106129, 62176139, 62276155), Natural Science Foundation of Shandong Province (No. ZR2021QF053, ZR2021ZD15) %and Chongqing Overseas Chinese Entrepreneurship and Innovation Support Program, and CAAI-Huawei MindSpore Open Fund.
\item \textbf{Conflict of interest} The author declares that he has no confict of interest. 
\item \textbf{Ethics approval} Not applicable.
\item \textbf{Consent to participate} Not applicable.
\item \textbf{Consent for publication} Not applicable.
\item \textbf{Availability of data and materials} Not applicable.
\item \textbf{Code availability} The code is now available at {\url{https://github.com/haolsun/VRI}}.
\item \textbf{Authors' contributions} 
H-Sun conceptualized the learning problem and provided the main idea. He also drafted the article.
Q-Wei completed main experiments and provided the analysis of experimental results.
L-Feng provided the theoretical guarantee for the learning algorithm.
F-Liu and H-Fan contributed to participating in discussions of the algorithm and experimental designs.
Y-Hu and Y-Yin provided funding supports, and Y-Hu approved the final version of the article.
\end{itemize}

\bibliography{sn-bibliography}% common bib file

\end{document}